\documentclass{article}
\usepackage{amsmath,amsfonts,amssymb,amsthm}
\usepackage{authblk}
\usepackage{fullpage}
\usepackage[usenames,dvipsnames]{color}
\usepackage{tikz}
\usetikzlibrary{calc,shapes}
\usepackage{algorithm,algorithmic}
\usepackage{wrapfig}
\def\E{\mathbb{E}}
\def\R{\mathbb{R}}
\def\Sig{\varSigma}
\def\wh{\widehat}
\def\tl{\tilde}

\def\Forest{\mathcal{F}}
\def\Tree{\mathbb{T}}
\def\Vars{\mathcal{V}_{\Tree}}
\def\Vobs{\mathcal{V}_{\operatorname{obs}}}
\def\Vhid{\mathcal{V}_{\operatorname{hid}}}
\def\Z{\mathcal{Z}}

\def\Vb{\bar{\mathcal{V}}}
\def\Sb{\bar{\mathcal{S}}}

\def\Subtree{\mathcal{T}}
\def\Supertree{\mathcal{ST}}
\def\Comps{\mathcal{C}}
\def\Roots{\mathcal{R}}
\def\Leaves{\mathcal{L}}

\renewcommand{\algorithmicrequire}{\textbf{Input:}}
\renewcommand{\algorithmicensure}{\textbf{Output:}}
\def\SQT{\mathsf{SpectralQuartetTest}}
\def\Mergeable{\mathsf{Mergeable}}
\def\Relationship{\mathsf{Relationship}}

\def\gammamin{\gamma_{\min}}
\def\gammamax{\gamma_{\max}}

\def\rhomax{\rho_{\max}}
\def\epsmin{\epsilon_{\min}}

\def\Des{{\operatorname{Descendants}_\Tree}}
\def\Children{{\operatorname{Children}_\Tree}}

\DeclareMathOperator{\range}{range}
\DeclareMathOperator{\rank}{rank}
\DeclareMathOperator{\tr}{tr}

\newcommand\A[1]{\ensuremath{A_{(#1)}}}
\newcommand\C[1]{\ensuremath{C_{(#1)}}}

\def\veps{\varepsilon}
\def\vsig{\varsigma}

\def\Lb{\bar{\mathcal{L}}}

\def\Det{{\textstyle\operatorname{det}}}

\newcommand\thirdlevel[1]{\medskip\noindent {\bf #1}}

\newtheorem{theorem}{Theorem}
\newtheorem{lemma}{Lemma}
\newtheorem{claim}{Claim}

\newtheorem{proposition}{Proposition}
\theoremstyle{remark}
\newtheorem{remark}{Remark}

\theoremstyle{definition}
\newtheorem{condition}{Condition}

\title{Spectral Methods for Learning Multivariate Latent Tree Structure}

\author[1]{Animashree Anandkumar}
\author[2]{Kamalika Chaudhuri}
\author[3]{Daniel Hsu}
\author[3,4]{Sham M.~Kakade}
\author[5]{\mbox{Le Song}}
\author[6]{Tong Zhang}
\affil[1]{Department of Electrical Engineering and Computer Science, UC Irvine}
\affil[2]{Department of Computer Science and Engineering, UC San Diego}
\affil[3]{Microsoft Research New England}
\affil[4]{Department of Statistics, Wharton School, University of Pennsylvania}
\affil[5]{Machine Learning Department, Carnegie Mellon University}
\affil[6]{Department of Statistics, Rutgers University}

\begin{document}

\maketitle

{\def\thefootnote{}%
\newcommand\email[1]{{\small\texttt{#1}}}%
\footnotetext{E-mail:
\email{a.anandkumar@uci.edu},
\email{kamalika@cs.ucsd.edu},
\email{dahsu@microsoft.com},
\email{skakade@microsoft.com},
\email{lesong@cs.cmu.edu},
\email{tzhang@stat.rutgers.edu}}}

\begin{abstract}
This work considers the problem of learning the structure of multivariate
linear tree models, which include a variety of directed tree graphical
models with continuous, discrete, and mixed latent variables such as
linear-Gaussian models, hidden Markov models, Gaussian mixture models, and
Markov evolutionary trees.
The setting is one where we only have samples from certain observed
variables in the tree, and our goal is to estimate the tree structure
(\emph{i.e.}, the graph of how the underlying hidden variables are
connected to each other and to the observed variables).
We propose the Spectral Recursive Grouping algorithm, an efficient and
simple bottom-up procedure for recovering the tree structure from
independent samples of the observed variables.
Our finite sample size bounds for exact recovery of the tree structure
reveal certain natural dependencies on underlying statistical and
structural properties of the underlying joint distribution.
Furthermore, our sample complexity guarantees have no explicit dependence
on the dimensionality of the observed variables, making the algorithm
applicable to many high-dimensional settings.
At the heart of our algorithm is a spectral quartet test for determining
the relative topology of a quartet of variables from second-order
statistics.
\end{abstract}

\section{Introduction}

Graphical models are a central tool in modern machine learning
applications, as they provide a natural methodology for succinctly
representing high-dimensional distributions.
As such, they have enjoyed much success in various AI and machine learning
applications such as natural language processing, speech recognition,
robotics, computer vision, and bioinformatics. 

The main statistical challenges associated with graphical models include
estimation and inference.
While the body of techniques for probabilistic inference in graphical
models is rather rich~\cite{Wainwright&Jordan:08NOW}, current methods for
tackling the more challenging problems of parameter and structure
estimation are less developed and understood, especially in the presence of
latent (hidden) variables.
The problem of parameter estimation involves determining the model
parameters from samples of certain observed variables.
Here, the predominant approach is the expectation maximization (EM)
algorithm, and only rather recently is the understanding of this algorithm
improving~\cite{DS07,kmeans}.
The problem of structure learning is to estimate the underlying graph of
the graphical model.  
In general, structure learning is NP-hard and becomes even more challenging
when some variables are unobserved~\cite{CHM04}.
The main approaches for structure estimation are either greedy or local
search approaches~\cite{Chow&Liu:68IT,FNP99} or, more recently, based on
convex relaxation~\cite{Ravikumar&etal:08Stat}.

This work focuses on learning the structure of multivariate latent tree
graphical models.
Here, the underlying graph is a directed tree (\emph{e.g.}, hidden Markov
model, binary evolutionary tree), and only samples from a set of
(multivariate) observed variables (the leaves of the tree) are available
for learning the structure.
Latent tree graphical models are relevant in many applications, ranging
from computer vision, where one may learn object/scene structure from the
co-occurrences of objects to aid image understanding~\cite{choi_cvpr10}; to
phylogenetics, where the central task is to reconstruct the tree of life
from the genetic material of surviving species~\cite{Durbin:book}.

Generally speaking, methods for learning latent tree structure exploit
structural properties afforded by the tree that are revealed through
certain statistical tests over every choice of four variables in the tree.
These \emph{quartet tests}, which have origins in structural equation
modeling~\cite{Wishart28,Bollen89}, are hypothesis tests of the relative
configuration of four (possibly non-adjacent) nodes/variables in the tree
(see Figure~\ref{fig:topologies}); they are also related to the \emph{four
point condition} associated with a corresponding additive tree metric
induced by the distribution~\cite{Buneman71}.
Some early methods for learning tree structure are based on the use of
\emph{exact} correlation statistics or distance measurements~(\emph{e.g.},
\cite{PT86,SN87}).
Unfortunately, these methods ignore the crucial aspect of estimation error,
which ultimately governs their sample complexity.
Indeed, this (lack of) robustness to estimation error has been quantified
for various algorithms (notably, for the popular Neighbor Joining
algorithm~\cite{ESSW99b,LC06}), and therefore serves as a basis for
comparing different methods.
Subsequent work in the area of mathematical phylogenetics has focused on
the sample complexity of evolutionary tree
reconstruction~\cite{ESSW99a,ESSW99b,Mos04,DMR11}.
The basic model there corresponds to a directed tree over discrete random
variables, and much of the recent effort deals exclusively in the regime
for a certain model parameter (the Kesten-Stigum regime~\cite{KS66}) that
allows for a sample complexity that is polylogarithmic in the number of
leaves, as opposed to polynomial~\cite{Mos04,DMR11}.
Finally, recent work in machine learning has developed structure learning
methods for latent tree graphical models that extend beyond the discrete
distributions of evolutionary trees~\cite{Choi&etal:10JMLR}, thereby
widening their applicability to other problem domains.

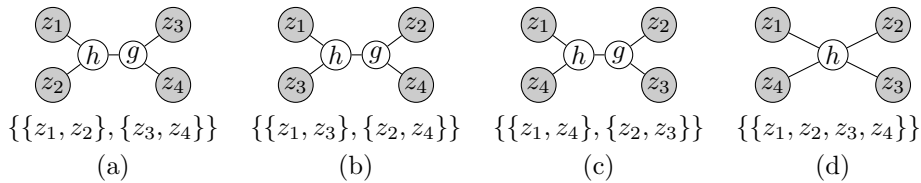
\begin{figure}
\begin{center}
\begin{tabular}{cccc}
\begin{tikzpicture}
  [
    scale=0.8,
    observed/.style={circle,inner sep=0.3mm,draw=black,fill=black!20},
    hidden/.style={circle,inner sep=0.3mm,draw=black}
  ]
  \node [observed,name=z1] at (-1,0.5) {$z_1$};
  \node [observed,name=z2] at (-1,-0.5) {$z_2$};
  \node [observed,name=z3] at (1,0.5) {$z_3$};
  \node [observed,name=z4] at (1,-0.5) {$z_4$};
  \node [hidden,name=h] at ($(-1/3,0)$) {$h$};
  \node [hidden,name=g] at ($(1/3,0)$) {$g$};
  \draw [-] (z1) to (h);
  \draw [-] (z2) to (h);
  \draw [-] (z3) to (g);
  \draw [-] (z4) to (g);
  \draw [-] (h) to (g);
\end{tikzpicture}
&
\begin{tikzpicture}
  [
    scale=0.8,
    observed/.style={circle,inner sep=0.3mm,draw=black,fill=black!20},
    hidden/.style={circle,inner sep=0.3mm,draw=black}
  ]
  \node [observed,name=z1] at (-1,0.5) {$z_1$};
  \node [observed,name=z3] at (-1,-0.5) {$z_3$};
  \node [observed,name=z2] at (1,0.5) {$z_2$};
  \node [observed,name=z4] at (1,-0.5) {$z_4$};
  \node [hidden,name=h] at ($(-1/3,0)$) {$h$};
  \node [hidden,name=g] at ($(1/3,0)$) {$g$};
  \draw [-] (z1) to (h);
  \draw [-] (z3) to (h);
  \draw [-] (z2) to (g);
  \draw [-] (z4) to (g);
  \draw [-] (h) to (g);
\end{tikzpicture}
&
\begin{tikzpicture}
  [
    scale=0.8,
    observed/.style={circle,inner sep=0.3mm,draw=black,fill=black!20},
    hidden/.style={circle,inner sep=0.3mm,draw=black}
  ]
  \node [observed,name=z1] at (-1,0.5) {$z_1$};
  \node [observed,name=z4] at (-1,-0.5) {$z_4$};
  \node [observed,name=z2] at (1,0.5) {$z_2$};
  \node [observed,name=z3] at (1,-0.5) {$z_3$};
  \node [hidden,name=h] at ($(-1/3,0)$) {$h$};
  \node [hidden,name=g] at ($(1/3,0)$) {$g$};
  \draw [-] (z1) to (h);
  \draw [-] (z4) to (h);
  \draw [-] (z2) to (g);
  \draw [-] (z3) to (g);
  \draw [-] (h) to (g);
\end{tikzpicture}
&
\begin{tikzpicture}
  [
    scale=0.8,
    observed/.style={circle,inner sep=0.3mm,draw=black,fill=black!20},
    hidden/.style={circle,inner sep=0.3mm,draw=black}
  ]
  \node [observed,name=z1] at (-1,0.5) {$z_1$};
  \node [observed,name=z2] at (-1,-0.5) {$z_4$};
  \node [observed,name=z3] at (1,0.5) {$z_2$};
  \node [observed,name=z4] at (1,-0.5) {$z_3$};
  \node [hidden,name=h] at (0,0) {$h$};
  \draw [-] (z1) to (h);
  \draw [-] (z2) to (h);
  \draw [-] (z3) to (h);
  \draw [-] (z4) to (h);
\end{tikzpicture}
\\
$\{\{z_1,z_2\},\{z_3,z_4\}\}$
& $\{\{z_1,z_3\},\{z_2,z_4\}\}$
& $\{\{z_1,z_4\},\{z_2,z_3\}\}$
& $\{\{z_1,z_2,z_3,z_4\}\}$
\vspace{1mm}
\\
(a) & (b) & (c) & (d)
\end{tabular}
\vspace{-3mm}
\end{center}
\caption{The four possible (undirected) tree topologies over leaves
$\{z_1,z_2,z_3,z_4\}$.}
\label{fig:topologies}
\end{figure}

This work extends beyond previous studies, which have focused on latent
tree models with either discrete or scalar Gaussian variables, by directly
addressing the multivariate setting where hidden and observed nodes may be
random vectors rather than scalars.
The generality of our techniques allows us to handle a much wider class of
distributions than before, both in terms of the conditional independence
properties imposed by the models (\emph{i.e.}, the random vector associated
with a node need not follow a distribution that corresponds to a tree
model), as well as other characteristics of the node distributions
(\emph{e.g.}, some nodes in the tree could have discrete state spaces and
others continuous, as in a Gaussian mixture model).

We propose the \emph{Spectral Recursive Grouping} algorithm for learning
multivariate latent tree structure.
The algorithm has at its core a multivariate \emph{spectral quartet test},
which extends the classical quartet tests for scalar variables by applying
spectral techniques from multivariate statistics (specifically canonical
correlation analysis~\cite{Bartlett38,MW80}).
Spectral methods have enjoyed recent success in the context of parameter
estimation~\cite{MR06,HKZ09,GordonHMM,HMM_kernel}; our work shows that they
are also useful for structure learning.
We use the spectral quartet test in a simple modification of the recursive
grouping algorithm of~\cite{Choi&etal:10JMLR} to perform the tree
reconstruction.
The algorithm is essentially a robust method for reasoning about the
results of quartet tests (viewed simply as hypothesis tests); the tests
either confirm or reject hypotheses about the relative topology over
quartets of variables.
By carefully choosing which tests to consider and properly interpreting
their results, the algorithm is able to recover the correct latent tree
structure (with high probability) in a provably efficient manner, in terms
of both computational and sample complexity.
The recursive grouping procedure is similar to the \emph{short quartet
method} from phylogenetics~\cite{ESSW99b}, which also guarantees efficient
reconstruction in the context of evolutionary trees.
However, our method and analysis applies to considerably more general
high-dimensional settings; for instance, our sample complexity bound is
given in terms of natural correlation conditions that generalize the more
restrictive \emph{effective depth} conditions of previous
works~\cite{ESSW99b,Choi&etal:10JMLR}.
Finally, we note that while we do not directly address the question of
parameter estimation, provable parameter estimation methods may derived
using the spectral techniques from~\cite{MR06,HKZ09}.

\section{Preliminaries} \label{section:prelim}



\subsection{Latent variable tree models}

Let $\Tree$ be a connected, directed tree graphical model with leaves
$\Vobs := \{ x_1, x_2, \dotsc, x_n \}$ and internal nodes $\Vhid := \{ h_1,
h_2, \dotsc, h_m \}$ such that every node has at most one parent.
The leaves are termed the \emph{observed variables} and the internal nodes
\emph{hidden variables}.
Note that all nodes in this work generally correspond to multivariate
random vectors; we will abuse terminology and still refer to these random
vectors as random variables.
For any $h \in \Vhid$, let $\Children(h) \subseteq \Vars$ denote the
children of $h$ in $\Tree$.

Each observed variable $x \in \Vobs$ is modeled as random vector in $\R^d$,
and each hidden variable $h \in \Vhid$ as a random vector in $\R^k$.
The joint distribution over all the variables $\Vars := \Vobs \cup \Vhid$
is assumed satisfy conditional independence properties specified by the
tree structure over the variables.
Specifically, for any disjoint subsets $V_1, V_2, V_3 \subseteq \Vars$ such
that $V_3$ separates $V_1$ from $V_2$ in $\Tree$, the variables in $V_1$
are conditionally independent of those in $V_2$ given $V_3$.

\subsection{Structural and distributional assumptions}

The class of models considered are specified by the following structural and
distributional assumptions.
\begin{condition}[Linear conditional means] \label{cond:linear}
Fix any hidden variable $h \in \Vhid$.
For each hidden child $g \in \Children(h) \cap \Vhid$, there
exists a matrix $\A{g|h} \in \R^{k \times k}$ such that
\[ \E[g|h] = \A{g|h} h ; \]
and for each observed child $x \in \Children(h) \cap \Vobs$,
there exists a matrix $\C{x|h} \in \R^{d \times k}$ such that
\[ \E[x|h] = \C{x|h} h . \]
\end{condition}
We refer to the class of tree graphical models satisfying
Condition~\ref{cond:linear} as \emph{linear tree models}.
Such models include a variety of continuous and discrete tree distributions
(as well as hybrid combinations of the two, such as Gaussian mixture
models) which are widely used in practice.
Continuous linear tree models include linear-Gaussian models and Kalman
filters.
In the discrete case, suppose that the observed variables take on $ d$
values, and hidden variables take $k$ values.
Then, each variable is represented by a binary vector in $\{ 0, 1 \}^s$,
where $s=d$ for the observed variables and $s=k$ for the hidden variables
(in particular, if the variable takes value $i$, then the corresponding
vector is the $i$-th coordinate vector), and any conditional distribution
between the variables is represented by a linear relationship.
Thus, discrete linear tree models include discrete hidden Markov
models~\cite{HKZ09} and Markovian evolutionary trees~\cite{MR06}.

In addition to the linearity, the following conditions are assumed in
order to recover the hidden tree structure.
For any matrix $M$, let $\sigma_t(M)$ denote its $t$-th largest singular
value.
\begin{condition}[Rank condition] \label{cond:full-rank}
The variables in $\Vars = \Vhid \cup \Vobs$ obey the following rank
conditions.
\begin{enumerate}
\item For all $h \in \Vhid$, $\E[hh^\top]$ has rank $k$ (\emph{i.e.},
$\sigma_k(\E[hh^\top]) > 0$).

\item For all $h \in \Vhid$ and hidden child $g \in \Children(h) \cap
\Vhid$, $\A{g|h}$ has rank $k$.

\item For all $h \in \Vhid$ and observed child $x \in \Children(h) \cap
\Vobs$, $\C{x|h}$ has rank $k$.

\end{enumerate}
\end{condition}
The rank condition is a generalization of parameter identifiability
conditions in latent variable models~\cite{Allman:09Stat,MR06,HKZ09} which
rules out various (provably) hard instances in discrete variable
settings~\cite{MR06}. 

%

\begin{condition}[Non-redundancy condition] \label{cond:non-redundancy}
Each hidden variable has at least three neighbors.
Furthermore, there exists $\rhomax^2 > 0$ such that for each pair of
distinct hidden variables $h,g \in \Vhid$,
\[
\frac{\det(\E[hg^{\top}])^2}{\det(\E[hh^\top]) \det(\E[gg^\top])}
\leq \rhomax^2 < 1 .
\]
\end{condition}
The requirement for each hidden node to have three neighbors is natural;
otherwise, the hidden node can be eliminated.
The quantity $\rhomax$ is a natural multivariate generalization of
correlation.
First, note that $\rhomax \leq 1$, and that if $\rhomax=1$ is achieved with
some $h$ and $g$, then $h$ and $g$ are completely correlated, implying the
existence of a deterministic map between hidden nodes $h$ and $g$; hence
simply merging the two nodes into a single node $h$ (or $g$) resolves this
issue.
Therefore the non-redundancy condition simply means that any two hidden
nodes $h$ and $g$ cannot be further reduced to a single node. 
Clearly, this condition is necessary for the goal of identifying the
correct tree structure, and it is satisfied as soon as $h$ and $g$ have
limited correlation in just a single direction.
Previous works~\cite{PT86,Pearl:book} show that an analogous condition
ensures identifiability for \emph{general} latent tree models (and in fact,
the conditions are identical in the Gaussian case).
Condition~\ref{cond:non-redundancy} is therefore a generalization of this
condition suitable for the multivariate setting.

Our learning guarantees also require a correlation condition that
generalize the explicit depth conditions considered in the phylogenetics
literature~\cite{ESSW99b,MR06}.
To state this condition, first define $\Forest_h$ to be the set of subtrees
of that remain after a hidden variable $h \in \Vhid$ is removed from
$\Tree$ (see Figure~\ref{fig:forest}).
Also, for any subtree $\Subtree'$ of $\Tree$, let $\Vobs[\Subtree']
\subseteq \Vobs$ be the observed variables in $\Subtree'$.

\begin{figure}
\begin{center}
\includegraphics[width=0.35\textwidth]{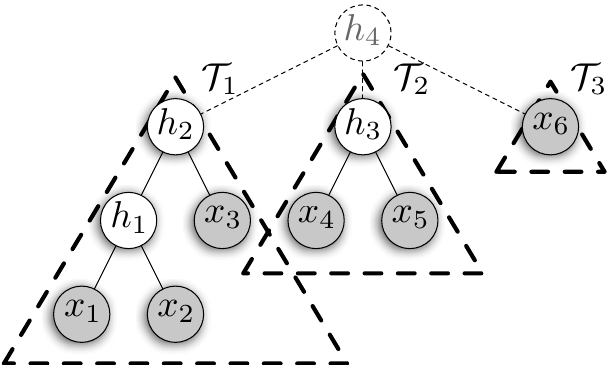}
\end{center}
\caption{Set of trees $\Forest_{h_4} = \{ \Subtree_1, \Subtree_2, \Subtree_3 \}$
obtained if $h_4$ is removed.
}
\label{fig:forest}
\end{figure}

\begin{condition}[Correlation condition] \label{cond:correlation}
There exists $\gammamin > 0$ such that for all hidden variables $h \in
\Vhid$ and all triples of subtrees $\{\Subtree_1,\Subtree_2,\Subtree_3\}
\subseteq \Forest_h$ in the forest obtained if $h$ is removed from $\Tree$,
\begin{equation*}
\max_{x_1 \in \Vobs[\Subtree_1], x_2 \in \Vobs[\Subtree_2], x_3 \in
\Vobs[\Subtree_3]}
\min_{\{i,j\} \subset \{1,2,3\}}
\sigma_k(\E[x_ix_j^\top]) \geq \gammamin
.
\end{equation*}
\end{condition}
The quantity $\gammamin$ is related to the \emph{effective depth} of
$\Tree$, which is the maximum graph distance between a hidden variable and
its closest observed variable~\cite{ESSW99b,Choi&etal:10JMLR}.
The effective depth is at most logarithmic in the number of variables (as
achieved by a complete binary tree), though it can also be a constant if every
hidden variable is close to an observed variable (\emph{e.g.}, in a hidden
Markov model, the effective depth is $1$, even though the true depth, or
diameter, is $m+1$).
If the matrices giving the (conditionally) linear relationship between
neighboring variables in $\Tree$ are all well-conditioned, then $\gammamin$
is at worst exponentially small in the effective depth, and therefore at
worst polynomially small in the number of variables.

Finally, also define
\[ \gammamax
:= \max_{\{x_1,x_2\} \subseteq \Vobs} \{ \sigma_1(\E[x_1x_2^\top]) \}
\]
to be the largest spectral norm of any second-moment matrix between
observed variables. Note $\gammamax\leq 1$ in the discrete case, and,
in the continuous case, $\gammamax\leq 1$ if each observed random
vector is in isotropic position.

In this work, the Euclidean norm of a vector $x$ is denoted by $\|x\|$, and
the (induced) spectral norm of a matrix $A$ is denoted by $\|A\|$,
\emph{i.e.}, $\|A\| := \sigma_1(A) = \sup \{ \|Ax\| \colon \|x\|=1\}$.


\section{Spectral quartet tests} \label{sec:quartettest}

This section describes
the core of our learning algorithm, a spectral quartet test that determines
topology of the subtree induced by four observed
variables $\{ z_1, z_2, z_3, z_4 \}$.
There are four possibilities for the induced subtree, as shown in
Figure~\ref{fig:topologies}.
Our quartet test either returns the correct induced subtree among
possibilities in Figure~\ref{fig:topologies}(a)--(c); or it outputs $\bot$ to indicate abstinence.
If the test returns $\bot$, then no guarantees are provided on the
induced subtree topology. If it does return a subtree, then the output
is guaranteed to be the correct induced subtree (with
high probability). 




The quartet test proposed is described in Algorithm~\ref{alg:quartettest}
($\SQT$).
The notation $[a]_{+}$ denotes $\max\{0, a\}$ and $[t]$ (for an
integer $t$) denotes the set $\{1,2,\dotsc,t\}$.

The quartet test is defined with respect to four observed variables $\Z :=
\{z_1,z_2,z_3,z_4\}$.
For each pair of variables $z_i$ and $z_j$, it takes as input an empirical
estimate $\hat \Sig_{i,j}$ of the second-moment matrix $\E[z_i
z_j^{\top}]$, and confidence bound parameters $\Delta_{i,j}$ 
which are functions of $N$, the number of
samples used to compute the $\hat \Sig_{i,j}$'s, a confidence parameter $\delta$, and of properties of the distributions of $z_i$ and
$z_j$. In practice, one uses a single threshold $\Delta$ for all
pairs, which is tuned by the algorithm. Our theoretical analysis also
applies to this case. 
The output of the test is either $\bot$ or a \emph{pairing} of the
variables $\{\{z_i,z_j\},\{z_{i'},z_{j'}\}\}$.
For example, if the output is the pairing is $\{\{z_1, z_2\}, \{z_3, z_4\}\}$, then 
Figure~\ref{fig:topologies}(a) is the output topology.

\begin{algorithm}[t]
\caption{$\SQT$ on observed variables $\{z_1,z_2,z_3,z_4\}$.}
\label{alg:quartettest}
\begin{algorithmic}[1]


\REQUIRE
For each pair $\{i,j\} \subset \{1,2,3,4\}$,
an empirical estimate $\hat \Sig_{i,j}$ of the second-moment matrix $\E[z_i
z_j^{\top}]$ and a corresponding confidence parameter $\Delta_{i,j} > 0$.


\ENSURE Either a pairing $\{\{z_i,z_j\},\{z_{i'},z_{j'}\}\}$ or $\bot$.

 

%
%

\IFPARTS{there exists a partition of $\{z_1,z_2,z_3,z_4\} = \{z_i,z_j\} \cup
    \{z_{i'},z_{j'}\}$ such that
    \[
     \prod_{s=1}^k
     [\sigma_s(\hat\Sig_{i,j}) - \Delta_{i,j}]_+
     [\sigma_s(\hat\Sig_{i',j'}) - \Delta_{i',j'}]_+
     >
     \prod_{s=1}^k
     (\sigma_s(\hat\Sig_{i',j}) + \Delta_{i',j})
     (\sigma_s(\hat\Sig_{i,j'}) + \Delta_{i,j'})\]
     }
    {return the pairing $\{\{z_i,z_j\}, \{z_{i'},z_{j'}\}\}$.}

\ELSEPART{return $\bot$.}

\end{algorithmic}
\end{algorithm}
Even though the configuration in Figure~\ref{fig:topologies}(d) is a
possibility, the spectral quartet test never returns
$\{\{z_1,z_2,z_3,z_4\}\}$, as there is no correct
pairing of $\Z$.
The topology $\{\{z_1,z_2,z_3,z_4\}\}$ can be viewed as a degenerate case
of $\{\{z_1,z_2\},\{z_3,z_4\}\}$ (say) where the hidden variables $h$ and
$g$ are deterministically identical, and
Condition~\ref{cond:non-redundancy} fails to hold with respect to $h$ and
$g$.

%

\subsection{Properties of the spectral quartet test}

\thirdlevel{With exact second moments:}
The spectral quartet test is motivated by the following lemma, which shows
the relationship between the singular values of second-moment matrices of
the $z_i$'s and the induced topology among them in the latent tree.
Let $\Det_k(M) := \prod_{s=1}^k \sigma_s(M)$ denote the product of the $k$
largest singular values of a matrix $M$.

\begin{lemma}[Perfect quartet test]\label{lemma:spectral-properties}
Suppose that the observed variables $\Z = \{ z_1, z_2, z_3, z_4 \}$ have the
true induced tree topology shown in Figure~\ref{fig:topologies}(a), and the
tree model satisfies Condition~\ref{cond:linear} and
Condition~\ref{cond:full-rank}.
Then
\begin{gather}
\frac{\Det_k(\E[z_1z_3^\top]) \Det_k(\E[z_2z_4^\top])}
{\Det_k(\E[z_1z_2^\top]) \Det_k(\E[z_3z_4^\top])}
=
\frac{\Det_k(\E[z_1z_4^\top]) \Det_k(\E[z_2z_3^\top])}
{\Det_k(\E[z_1z_2^\top]) \Det_k(\E[z_3z_4^\top])}
=
\frac{\det(\E[hg^\top])^2}{\det(\E[hh^\top]) \det(\E[gg^\top])}
\leq 1
\label{eqn:spectral-properties1} \\
\text{and} \quad
{\Det_k(\E[z_1z_3^\top]) \Det_k(\E[z_2z_4^\top])}
= {\Det_k(\E[z_1z_4^\top]) \Det_k(\E[z_2z_3^\top])}
.
\nonumber
\end{gather}
\end{lemma}

This lemma shows that given the true second-moment matrices and assuming
Condition~\ref{cond:non-redundancy}, the inequality
in~\eqref{eqn:spectral-properties1} becomes strict and thus can be used to
deduce the correct topology:
the correct pairing is $\{\{z_i,z_j\},\{z_{i'},z_{j'}\}\}$ if and only if
\[
\Det_k(\E[z_iz_j^\top]) \Det_k(\E[z_{i'}z_{j'}^\top])
> \Det_k(\E[z_{i'}z_j^\top]) \Det_k(\E[z_iz_{j'}^\top])
.
\]

\thirdlevel{Reliability:}
The next lemma shows that even if the singular values of $\E[z_iz_j^\top]$
are not known exactly, then with valid confidence intervals (that contain
these singular values) a robust test can be constructed which is reliable
in the following sense: if it does not output $\bot$, then the output
topology is indeed the correct topology.
\begin{lemma}[Reliability]\label{lemma:reliability}
Consider the setup of Lemma~\ref{lemma:spectral-properties}, and suppose
that Figure~\ref{fig:topologies}(a) is the correct topology.
If for all pairs $\{z_i, z_j\} \subset \Z$ and all $s \in [k]$,
$\sigma_s(\hat\Sig_{i,j}) - \Delta_{i,j}
\leq
\sigma_s(\E[z_iz_j^\top])
\leq
\sigma_s(\hat\Sig_{i,j}) + \Delta_{i,j}$,
and if $\SQT$ returns a pairing
$\{\{z_i,z_j\},\{z_{i'},z_{j'}\}\}$, then $\{\{z_i,z_j\},\{z_{i'},z_{j'}\}\} =
\{\{z_1,z_2\},\{z_3,z_4\}\}$.
\end{lemma}
In other words, the spectral quartet test never returns an incorrect
pairing as long as the singular values of $\E[z_iz_j^\top]$ lie in an
interval of length $2 \Delta_{i,j}$ around the singular values of $\hat
\Sig_{i,j}$.
The lemma below shows how to set
the $\Delta_{i,j}$s as a function of $N$, $\delta$ and properties of the
distributions of $z_i$ and $z_j$ so that this required event holds with
probability at least $1 - \delta$.
We remark that any valid confidence intervals may be used; the one
described below is particularly suitable when the observed variables are
high-dimensional random vectors.
\begin{lemma}[Confidence intervals] \label{lem:confinterval}
Let $\Z = \{ z_1, z_2, z_3, z_4 \}$ be four random vectors.
Let $\|z_i\| \leq M_i$ almost surely, and let $\delta \in (0,1/6)$.
If each empirical second-moment matrix $\hat\Sig_{i,j}$ is computed using
$N$ iid copies of $z_i$ and $z_j$, and if
\begin{align*}
 \bar{d}_{i,j} & :=
\frac{\E[\|z_i\|^2\|z_j\|^2] - \tr(\E[z_iz_j^\top]\E[z_iz_j^\top]^\top)}
{\max\{ \|\E[\|z_j\|^2 z_iz_i^\top]\|, \|\E[\|z_i\|^2 z_jz_j^\top]\| \}} , &
t_{i,j} & := 1.55 \ln(24\bar{d}_{i,j}/\delta) , \\
\lefteqn{
\Delta_{i,j} \geq
\sqrt{\frac{2\max\bigl\{
\bigl\|\E[\|z_j\|^2z_iz_i^\top]\bigr\|,
\bigl\|\E[\|z_i\|^2z_jz_j^\top]\bigr\| \bigr\} t_{i,j}}{N}} +
\frac{M_i M_j t_{i,j}}{3N}
,
}
\end{align*}
then with probability $1 - \delta$,
for all pairs $\{z_i, z_j\} \subset \Z$ and all $s \in [k]$,
\begin{equation} \label{eq:confidence-bounds-test}
\sigma_{s}(\hat \Sig_{i,j}) - \Delta_{i,j} \leq \sigma_s(\E[z_iz_j^\top])
\leq \sigma_s(\hat \Sig_{i,j}) + \Delta_{i,j} 
.
\end{equation}
\end{lemma}

\thirdlevel{Conditions for returning a correct pairing:}
The conditions under which $\SQT$
returns an induced topology (as opposed to $\bot$) are now
provided.


An important quantity in this analysis is the level of non-redundancy
between the hidden variables $h$ and $g$.
Let
\begin{equation} \label{eq:rho}
\rho^2 := \frac{\det(\E[hg^\top])^2}{\det(\E[hh^\top]) \det(\E[gg^\top])}
.
\end{equation}
If Figure~\ref{fig:topologies}(a) is the correct induced topology among
$\{z_1,z_2,z_3,z_4\}$, then the smaller $\rho$ is, the greater the
gap between $\Det_k(\E[z_1z_2^\top]) \Det_k(\E[z_3z_4^\top])$ and
either of $\Det_k(\E[z_1z_3^\top]) \Det_k(\E[z_2z_4^\top])$ and
$\Det_k(\E[z_1z_4^\top]) \Det_k(\E[z_2z_3^\top])$.
Therefore, $\rho$ also governs how small the $\Delta_{i,j}$ need to be for
the quartet test to return a correct pairing; this is quantified in
Lemma~\ref{lem:usefulrankk}.
Note that Condition~\ref{cond:non-redundancy} implies $\rho \leq \rhomax <
1$.
\begin{lemma}[Correct pairing] \label{lem:usefulrankk}
Suppose that
(i) the observed variables $\Z = \{ z_1, z_2, z_3, z_4 \}$ have the true
induced tree topology shown in Figure~\ref{fig:topologies}(a);
(ii) the tree model satisfies Condition~\ref{cond:linear},
Condition~\ref{cond:full-rank}, and $\rho < 1$ (where $\rho$ is defined
in~\eqref{eq:rho}),
and (iii) the confidence bounds in~\eqref{eq:confidence-bounds-test} hold
for all $\{i,j\}$ and all $s \in [k]$.
If
\[ \Delta_{i,j} < \frac{1}{8k} \cdot \min\Bigl\{1, \ \frac{1}{\rho}
- 1 \Bigr\} \cdot \min_{\{i,j\}} \{ \sigma_k(\E[z_iz_j^\top]) \} \]
for each pair $\{i,j\}$, then $\SQT$ returns the
correct pairing $\{ \{z_1, z_2\}, \{z_3, z_4\}\}$.
\end{lemma}

\section{The Spectral Recursive Grouping algorithm}
\label{sec:rg}

The Spectral Recursive Grouping algorithm, presented as
Algorithm~\ref{alg:rg}, uses the spectral quartet test discussed in the
previous section to estimate the structure of a multivariate latent tree
distribution from iid samples of the observed leaf variables.\footnote{To
simplify notation, we assume that the estimated second-moment matrices
$\wh\Sig_{x,y}$ and threshold parameters $\Delta_{x,y} \geq 0$ for all
pairs $\{x,y\} \subset \Vobs$ are globally defined.
In particular, we assume the spectral quartet tests use these quantities.}
The algorithm is a modification of the recursive grouping (RG) procedure
proposed in~\cite{Choi&etal:10JMLR}.
RG builds the tree in a bottom-up fashion, where the initial working set of
variables are the observed variables.
The variables in the working set always correspond to roots of disjoint
subtrees of $\Tree$ discovered by the algorithm.
(Note that because these subtrees are rooted, they naturally induce
parent/child relationships, but these may differ from those implied by the
edge directions in $\Tree$.)
In each iteration, the algorithm determines which variables in the working
set to combine.
If the variables are combined as siblings, then a new hidden variable is
introduced as their parent and is added to the working set, and its
children are removed.
If the variables are combined as neighbors (parent/child), then the child
is removed from the working set.
The process repeats until the entire tree is constructed.

Our modification of RG uses the spectral quartet tests from
Section~\ref{sec:quartettest} to decide which subtree roots in the current
working set to combine.
Note that because the test may return $\bot$ (a null result), our algorithm
uses the tests to \emph{rule out} possible siblings or neighbors among
variables in the working set---this is encapsulated in the subroutine
$\Mergeable$ (Algorithm~\ref{alg:mergeable}), which tests quartets of
observed variables (leaves) in the subtrees rooted at working set
variables.
For any pair $\{u,v\} \subseteq \Roots$ submitted to the subroutine (along
with the current working set $\Roots$ and leaf sets $\Leaves[\cdot]$):
\begin{itemize}
\item $\Mergeable$ returns false if there is evidence (provided by a
quartet test) that $u$ and $v$ should first be joined with different
variables ($u'$ and $v'$, respectively) before joining with each other; and

\item $\Mergeable$ returns true if no quartet test provides such evidence.

\end{itemize}
The subroutine is also used by the subroutine $\Relationship$
(Algorithm~\ref{alg:relationship}) which determines whether a candidate
pair of variables should be merged as neighbors (parent/child) or as
siblings:
essentially, to check if $u$ is a parent of $v$, it checks if $v$ is a
sibling of each child of $u$.
The use of unreliable estimates of long-range correlations is avoided by
only considering highly-correlated variables as candidate pairs to merge
(where correlation is measured using observed variables in their
corresponding subtrees as proxies).
This leads to a sample-efficient algorithm for recovering the hidden tree
structure.

\begin{algorithm}[t]
\caption{Spectral Recursive Grouping.}
\label{alg:rg}
\begin{algorithmic}[1]
  \REQUIRE
  Empirical second-moment matrices $\wh\Sig_{x,y}$ for all pairs $\{x,y\}
  \subset \Vobs$ computed from $N$ iid samples from the distribution over
  $\Vobs$;
  threshold parameters $\Delta_{x,y}$ for all pairs $\{x,y\} \subset \Vobs$.

  \ENSURE
  Tree structure $\wh\Tree$ or ``failure''.

    \LET $\Roots := \Vobs$, and for all $x \in \Roots$,
    $\Subtree[x] := \text{rooted single-node tree $x$}$ and $\Leaves[x] :=
    \{ x \}$.

    \WHILE{$|\Roots| > 1$}

      \LET pair $\{u,v\} \in \{ \{\tl{u},\tl{v}\} \subseteq \Roots :
      \Mergeable(\Roots,\Leaves[\cdot],\tl{u},\tl{v}) = \text{true} \}$ be
      such that $\max \{ \sigma_k(\wh\Sig_{x,y}) : (x,y) \in \Leaves[u]
      \times \Leaves[v] \}$ is maximized.
      If no such pair exists, then halt and return ``failure''.

      \LET $\mathtt{result} :=
      \Relationship(\Roots,\Leaves[\cdot],\Subtree[\cdot],u,v)$.

      \IF{$\mathtt{result} = \text{``siblings''}$}

        \STATE Create a new variable $h$, create subtree $\Subtree[h]$
        rooted at $h$ by joining $\Subtree[u]$ and $\Subtree[v]$ to $h$
        with edges $\{h,u\}$ and $\{h,v\}$, and set $\Leaves[h] :=
        \Leaves[u] \cup \Leaves[v]$.

        \STATE Add $h$ to $\Roots$, and remove $u$ and $v$ from $\Roots$.

      \ELSIF{$\mathtt{result} = \text{``$u$ is parent of $v$''}$}

        \STATE Modify subtree $\Subtree[u]$ by joining $\Subtree[v]$ to $u$
        with an edge $\{u,v\}$, and modify $\Leaves[u] := \Leaves[u] \cup
        \Leaves[v]$.

        \STATE Remove $v$ from $\Roots$.

      \ELSIF{$\mathtt{result} = \text{``$v$ is parent of $u$''}$}

        \STATE \COMMENT{Analogous to above case.}

      \ENDIF

    \ENDWHILE

    \STATE Return $\wh\Tree := \Subtree[h]$ where $\Roots = \{h\}$.

\end{algorithmic}
\end{algorithm}

\begin{algorithm}[t]
\caption{Subroutine $\Mergeable(\Roots,\Leaves[\cdot],u,v)$.}
\label{alg:mergeable}
\begin{algorithmic}[1]
  \REQUIRE
  Set of nodes $\Roots$;
  leaf sets $\Leaves[v]$ for all $v \in \Roots$;
  distinct $u,v \in \Roots$.

  \ENSURE true or false.

    \IFPART{there exists distinct $u',v' \in \Roots \setminus \{u,v\}$
    and $(x,y,x',y') \in \Leaves[u] \times \Leaves[v] \times \Leaves[u']
    \times \Leaves[v']$ s.t.\ $\SQT(\{x,y,x',y'\})$ returns
    $\{\{x,x'\},\{y,y'\}\}$ or $\{\{x,y'\},\{x',y\}\}$}
    {return false.}

    \ELSEPART{return true.}
\end{algorithmic}
\end{algorithm}

\begin{algorithm}[t]
\caption{Subroutine
$\Relationship(\Roots,\Leaves[\cdot],\Subtree[\cdot],u,v)$.}
\label{alg:relationship}
\begin{algorithmic}[1]
  \REQUIRE
  Set of nodes $\Roots$;
  leaf sets $\Leaves[v]$ for all $v \in \Roots$;
  rooted subtrees $\Subtree[v]$ for all $v \in \Roots$;
  distinct $u,v \in \Roots$.

  \ENSURE ``siblings'', ``$u$ is parent of $v$'' (``$u \to v$''), or ``$v$
  is parent of $u$'' (``$v \to u$'').

    \IFPART{$u$ is a leaf}
    {assert $u \not\to v$.}

    \IFPART{$v$ is a leaf}
    {assert $v \not\to u$.}

    \LET $\Roots[w] := (\Roots \setminus \{w\}) \cup \{ w' : \text{$w'$ is
    a child of $w$ in $\Subtree[w]$} \}$ for each $w \in \{u,v\}$.

    \IFPART{there exists child $u_1$ of $u$ in $\Subtree[u]$
    s.t.~$\Mergeable(\Roots[u],\Leaves[\cdot],u_1,v)\kern-1.5pt=\kern-1.5pt\text{false}$}
    {assert ``$u \not\to v$''.}

    \IFPART{there exists child $v_1$ of $v$ in $\Subtree[v]$
    s.t.~$\Mergeable(\Roots[v],\Leaves[\cdot],u,v_1)\kern-1.5pt=\kern-1.5pt\text{false}$}
    {assert ``$v \not\to u$''.}

    \IFPART{both ``$u \not\to v$'' and ``$v \not\to u$'' were asserted}
    {return ``siblings''.}

    \ELSIFPART{``$u \not\to v$'' was asserted}
    {return ``$v$ is parent of $u$'' (``$v \to u$'').}

    \ELSEPART{return ``$u$ is parent of $v$'' (``$u \to v$'').}

\end{algorithmic}
\end{algorithm}

The Spectral Recursive Grouping algorithm enjoys the following guarantee.
\begin{theorem} \label{theorem:rg}
Let $\eta \in (0,1)$. Assume the directed tree graphical model $\Tree$ over
variables (random vectors) $\Vars = \Vobs \cup \Vhid$ satisfies
Conditions~\ref{cond:linear}, \ref{cond:full-rank},
\ref{cond:non-redundancy}, and \ref{cond:correlation}.
Suppose the Spectral Recursive Grouping algorithm (Algorithm~\ref{alg:rg}) is provided $N$ independent
samples from the distribution over $\Vobs$, and uses parameters given by
\begin{align} \label{eq:thresh}
\Delta_{x_i,x_j}
& :=
\sqrt{\frac{2 B_{x_i,x_j} t_{x_i,x_j}}{N}}
+ \frac{M_{x_i}M_{x_j}t_{x_i,x_j}}{3N}
\end{align}
where
\begin{align*}
B_{x_i,x_j} & := \max\bigl\{
\bigl\|\E[\|x_i\|^2x_jx_j^\top]\bigr\|,
\bigl\|\E[\|x_j\|^2x_ix_i^\top]\bigr\|
\bigr\} , &
M_{x_i} & \geq \|x_i\| \quad \text{almost surely} , \\
\bar{d}_{x_i,x_j} & :=
\frac{\E[\|x_i\|^2 \|x_j\|^2] - \tr(\E[x_ix_j^\top]\E[x_jx_i^\top])}
{\max\bigl\{ \bigl\|\E[\|x_j\|^2x_ix_i^\top]\bigr\|,
\bigl\|\E[\|x_i\|^2x_jx_j^\top]\bigr\| \bigr\}} , &
t_{x_i,x_j} & := 4\ln(4\bar{d}_{x_i,x_j} n/\eta)
.
\end{align*}
Let
$B := \max_{ x_i,x_j \in \Vobs } \{ B_{x_i,x_j} \}$,
$M := \max_{ x_i \in \Vobs} \{ M_{x_i} \}$,
$t := \max_{ x_i,x_j \in \Vobs} \{ t_{x_i,x_j} \}$.
If
\[ N >
\frac{200 \cdot k^2 \cdot B \cdot t}
{\displaystyle\left( \frac{\gammamin^2}{\gammamax} \cdot (1-\rhomax) \right)^2}
+
\frac{7 \cdot k \cdot M^2 \cdot t}
{\displaystyle \frac{\gammamin^2}{\gammamax} \cdot (1-\rhomax)}
, \]
then with probability at least $1-\eta$, the Spectral Recursive Grouping
algorithm returns a tree $\wh\Tree$ with the same undirected graph
structure as $\Tree$.
\end{theorem}
Consistency is implied by the above theorem with an appropriate scaling of
$\eta$ with $N$.
The theorem reveals that the sample complexity of the algorithm depends
solely on intrinsic spectral properties of the distribution.
Note that there is no explicit dependence on the dimensions of the
observable variables, which makes the result applicable to high-dimensional
settings.


\subsubsection*{Acknowledgements}

Part of this work was completed while DH was at the Wharton School of the
University of Pennsylvania and at Rutgers University.
AA was supported by in part by the setup funds at UCI and the AFOSR Award
FA9550-10-1-0310.


\subsubsection*{References}
{\def\section*#1{}\small \bibliography{quartet} \bibliographystyle{plain}}

\begin{thebibliography}{10}

\bibitem{Allman:09Stat}
E.~S. Allman, C.~Matias, and J.~A. Rhodes.
\newblock {Identifiability of parameters in latent structure models with many
  observed variables}.
\newblock {\em The Annals of Statistics}, 37(6A):3099--3132, 2009.

\bibitem{Bartlett38}
M.~S. Bartlett.
\newblock Further aspects of the theory of multiple regression.
\newblock {\em Mathematical Proceedings of the Cambridge Philosophical
  Society}, 34:33--40, 1938.

\bibitem{Bollen89}
K.~Bollen.
\newblock {\em Structural Equation Models with Latent Variables}.
\newblock John Wiley \& Sons, 1989.

\bibitem{Buneman71}
P.~Buneman.
\newblock The recovery of trees from measurements of dissimilarity.
\newblock In F.~R. Hodson, D.~G. Kendall, and P.~Tautu, editors, {\em
  Mathematics in the Archaeological and Historical Sciences}, pages 387--395.
  1971.

\bibitem{kmeans}
K.~Chaudhuri, S.~Dasgupta, and A.~Vattani.
\newblock Learning mixtures of {G}aussians using the $k$-means algorithm, 2009.
\newblock arXiv:0912.0086.

\bibitem{CHM04}
D.~M. Chickering, D.~Heckerman, and C.~Meek.
\newblock Large-sample learning of {B}ayesian networks is {NP}-hard.
\newblock {\em Journal of Machine Learning Research}, 5:1287--1330, 2004.

\bibitem{choi_cvpr10}
M.~J. Choi, J.~J. Lim, A.~Torralba, and A.~S. Willsky.
\newblock Exploiting hierarchical context on a large database of object
  categories.
\newblock In {\em IEEE Conference on Computer Vision and Pattern Recognition},
  2010.

\bibitem{Choi&etal:10JMLR}
M.~J. Choi, V.~Tan, A.~Anandkumar, and A.~Willsky.
\newblock Learning latent tree graphical models.
\newblock {\em Journal of Machine Learning Research}, 12:1771--1812, 2011.

\bibitem{Chow&Liu:68IT}
C.~Chow and C.~Liu.
\newblock Approximating discrete probability distributions with dependence
  trees.
\newblock {\em IEEE Transactions on Information Theory}, 14(3):462--467, 1968.

\bibitem{DS07}
S.~Dasgupta and L.~Schulman.
\newblock A probabilistic analysis of {EM} for mixtures of separated, spherical
  {G}aussians.
\newblock {\em Journal of Machine Learning Research}, 8(Feb):203--226, 2007.

\bibitem{DMR11}
C.~Daskalakis, E.~Mossel, and S.~Roch.
\newblock Evolutionary trees and the {I}sing model on the {B}ethe lattice: A
  proof of {S}teel's conjecture.
\newblock {\em Probability Theory and Related Fields}, 149(1--2):149--189,
  2011.

\bibitem{Durbin:book}
R.~Durbin, S.~R. Eddy, A.~Krogh, and G.~Mitchison.
\newblock {\em Biological Sequence Analysis: Probabilistic Models of Proteins
  and Nucleic Acids}.
\newblock Cambridge University Press, 1999.

\bibitem{ESSW99a}
P.~L. Erd\"os, L.~A. Sz\'ekely, M.~A. Steel, and T.~J. Warnow.
\newblock A few logs suffice to build (almost) all trees ({I}).
\newblock {\em Random Structures and Algorithms}, 14:153--184, 1999.

\bibitem{ESSW99b}
P.~L. Erd\"os, L.~A. Sz\'ekely, M.~A. Steel, and T.~J. Warnow.
\newblock A few logs suffice to build (almost) all trees: Part {II}.
\newblock {\em Theoretical Computer Science}, 221:77--118, 1999.

\bibitem{FNP99}
N.~Friedman, I.~Nachman, and D.~Pe{\'e}r.
\newblock Learning {B}ayesian network structure from massive datasets: the
  ``sparse candidate'' algorithm.
\newblock In {\em Fifteenth Conference on Uncertainty in Artificial
  Intelligence}, 1999.

\bibitem{HKZ09}
D.~Hsu, S.~M. Kakade, and T.~Zhang.
\newblock A spectral algorithm for learning hidden {M}arkov models.
\newblock In {\em Twenty-Second Annual Conference on Learning Theory}, 2009.

\bibitem{HKZ11}
D.~Hsu, S.~M. Kakade, and T.~Zhang.
\newblock Dimension-free tail inequalities for sums of random matrices, 2011.
\newblock arXiv:1104.1672.

\bibitem{KS66}
H.~Kesten and B.~P. Stigum.
\newblock Additional limit theorems for indecomposable multidimensional
  galton-watson processes.
\newblock {\em Annals of Mathematical Statistics}, 37:1463--1481, 1966.

\bibitem{LC06}
M.~R. Lacey and J.~T. Chang.
\newblock A signal-to-noise analysis of phylogeny estimation by
  neighbor-joining: insufficiency of polynomial length sequences.
\newblock {\em Mathematical Biosciences}, 199(2):188--215, 2006.

\bibitem{Mos04}
E.~Mossel.
\newblock Phase transitions in phylogeny.
\newblock {\em Transactions of the American Mathematical Society},
  356(6):2379--2404, 2004.

\bibitem{MR06}
E.~Mossel and S.~Roch.
\newblock Learning nonsingular phylogenies and hidden {M}arkov models.
\newblock {\em Annals of Applied Probability}, 16(2):583--614, 2006.

\bibitem{MW80}
R.~J. Muirhead and C.~M. Waternaux.
\newblock Asymptotic distributions in canonical correlation analysis and other
  multivariate procedures for nonnormal populations.
\newblock {\em Biometrika}, 67(1):31--43, 1980.

\bibitem{Pearl:book}
J.~Pearl.
\newblock {\em Probabilistic Reasoning in Intelligent Systems---Networks of
  Plausible Inference}.
\newblock Morgan Kaufmann, 1988.

\bibitem{PT86}
J.~Pearl and M.~Tarsi.
\newblock Structuring causal trees.
\newblock {\em Journal of Complexity}, 2(1):60--77, 1986.

\bibitem{Ravikumar&etal:08Stat}
P.~Ravikumar, M.~J. Wainwright, and J.~Lafferty.
\newblock High-dimensional {I}sing model selection using $\ell_1$-regularized
  logistic regression.
\newblock {\em Annals of Statistics}, 38(3):1287--1319, 2010.

\bibitem{SN87}
N.~Saitou and M.~Nei.
\newblock The neighbor-joining method: A new method for reconstructing
  phylogenetic trees.
\newblock {\em Molecular Biology and Evolution}, 4:406--425, 1987.

\bibitem{GordonHMM}
S.~M. Siddiqi, B.~Boots, and G.~J. Gordon.
\newblock Reduced-rank hidden {M}arkov models.
\newblock In {\em Thirteenth International Conference on Artificial
  Intelligence and Statistics}, 2010.

\bibitem{HMM_kernel}
L.~Song, S.~M. Siddiqi, G.~J. Gordon, and A.~J. Smola.
\newblock Hilbert space embeddings of hidden {M}arkov models.
\newblock In {\em International Conference on Machine Learning}, 2010.

\bibitem{Wainwright&Jordan:08NOW}
M.~J. Wainwright and M.~I. Jordan.
\newblock Graphical models, exponential families, and variational inference.
\newblock {\em Foundations and Trends in Machine Learning}, 1(1-2):1--305,
  2008.

\bibitem{Wishart28}
J.~Wishart.
\newblock Sampling errors in the theory of two factors.
\newblock {\em British Journal of Psychology}, 19:180--187, 1928.

\end{thebibliography}


\appendix

\section{Sample-based confidence intervals for singular values}
\label{appendix:confidence-intervals}

We show how to derive confidence bounds for the singular values of
$\Sig_{i,j} := \E[z_iz_j^\top]$ for $\{i,j\} \subset \{1,2,3,4\}$ from $N$
iid copies of the random vectors $\{ z_1, z_2, z_3, z_4 \}$.
That is, we show how to set $\Delta_{i,j}$ so that, with high probability,
\[ \sigma_{s}(\hat \Sig_{i,j}) - \Delta_{i,j}
\ \leq \ \sigma_s(\Sig_{i,j}) \ \leq \
\sigma_s(\hat \Sig_{i,j}) + \Delta_{i,j} 
\]
for all $\{i,j\}$ and all $s \in [k]$.

We state exponential tail inequalities for the spectral norm of the
estimation error $\hat\Sig_{i,j} - \Sig_{i,j}$.
The first exponential tail inequality is stated for general random vectors
under Bernstein-type conditions, and the second is specific to random
vectors in the discrete setting.
\begin{lemma} \label{lemma:spectral-bernstein}
Let $z_i$ and $z_j$ be random vectors such that $\|z_i\| \leq M_i$ and
$\|z_j\| \leq M_j$ almost surely, and let
\[
\bar{d}_{i,j} :=
\frac{\E[\|z_i\|^2 \|z_j\|^2] - \tr(\Sig_{i,j}\Sig_{i,j}^\top)}
{\max\bigl\{ \bigl\|\E[\|z_j\|^2z_iz_i^\top]\bigr\|,
\bigl\|\E[\|z_i\|^2z_jz_j^\top]\bigr\| \bigr\}}
\leq \max\{\dim(z_i), \dim(z_j)\}
.
\]
Let $\Sig_{i,j} := \E[z_iz_j^\top]$ and let $\hat\Sig_{i,j}$ be the empirical
average of $N$ independent copies of $z_iz_j^\top$.
Pick any $t > 0$.
With probability at least $1-4\bar{d}_{i,j}t(e^t-t-1)^{-1}$,
\[
\Bigl\| \hat\Sig_{i,j} - \Sig_{i,j} \Bigr\| \leq \sqrt{\frac{2\max\bigl\{
\bigl\|\E[\|z_j\|^2z_iz_i^\top]\bigr\|,
\bigl\|\E[\|z_i\|^2z_jz_j^\top]\bigr\| \bigr\} t}{N}} +
\frac{M_iM_jt}{3N} 
.
\]
\end{lemma}
\begin{remark}
For any $\delta \in (0,1/6)$, we have $4\bar{d}_{i,j}t(e^t-t-1)^{-1} \leq
\delta$ provided that $t \geq 1.55 \ln(4\bar{d}_{i,j}/\delta)$.
\end{remark}
\begin{proof}
Define the random matrix
\[ Z := \begin{bmatrix}
& z_iz_j^\top \\
z_jz_i^\top &
\end{bmatrix}
.
\]
Let $Z_1,\dotsc,Z_N$ be independent copies of $Z$.
Then
\[
\Pr\left[ \Bigl\|\hat\Sig_{i,j} - \Sig_{i,j}\Bigr\| > t \right]
= \Pr\left[ \biggl\|\frac1N \sum_{\ell=1}^N Z_\ell - \E[Z]\biggr\| > t
\right]
.
\]
Note that
\[
\E[Z^2]
= \E \begin{bmatrix}
\|z_j\|^2 z_iz_i^\top & \\
& \|z_i\|^2 z_jz_j^\top
\end{bmatrix}
\]
so by convexity,
\begin{align*}
\bigl\|\E[Z^2] - \E[Z]^2\bigr\|
& \leq \bigl\|\E[Z^2]\bigr\| \\
& \leq \max\bigl\{
\bigl\|\E[\|z_j\|^2z_iz_i^\top]\bigr\|,
\bigl\|\E[\|z_i\|^2z_jz_j^\top]\bigr\|
\bigr\}
\end{align*}
and
\begin{align*}
\tr(\E[Z^2] - \E[Z]^2)
& =
\tr(\E[\|z_j\|^2 z_iz_i^\top]) + 
\tr(\E[\|z_i\|^2 z_jz_j^\top])
- \tr(\Sig_{i,j}\Sig_{i,j}^\top)
- \tr(\Sig_{i,j}^\top\Sig_{i,j})
\\
& = 2\Bigl(\E[\|z_i\|^2 \|z_j\|^2] -
\tr(\Sig_{i,j}\Sig_{i,j}^\top) \Bigr)
.
\end{align*}
Moreover,
\[
\|Z\| \leq \|z_i\| \|z_j\| \leq M_i M_j
.
\]
By the matrix Bernstein inequality~\cite{HKZ11}, for any $t > 0$,
\begin{multline*}
\Pr\left[ \Bigl\| \hat\Sig_{i,j} - \Sig_{i,j} \Bigr\| >
\sqrt{\frac{2
\left( \max\bigl\{ \bigl\|\E[\|z_j\|^2z_iz_i^\top]\bigr\|,
\bigl\|\E[\|z_i\|^2z_jz_j^\top]\bigr\| \bigr\} \right) t}{N}}
+ \frac{M_iM_jt}{3N} \right]
\\
\leq 2
\cdot \frac{2\Bigl(\E[\|z_i\|^2 \|z_j\|^2] - \tr(\Sig_{i,j}\Sig_{i,j}^\top)
\Bigr)}
{\max\bigl\{ \bigl\|\E[\|z_j\|^2z_iz_i^\top]\bigr\|,
\bigl\|\E[\|z_i\|^2z_jz_j^\top]\bigr\| \bigr\}}
\cdot t(e^t-t-1)^{-1}
= 4\bar{d}_{i,j} t(e^t-t-1)^{-1}
.
\end{multline*}
The claim follows.
\end{proof}

In the case of discrete random variables (modeled as random vectors as
described in Section~\ref{section:prelim}), the following lemma
from~\cite{HKZ09} can give a tighter exponential tail inequality.
\begin{lemma}[\cite{HKZ09}] \label{lemma:frobenius-discrete}
Let $z_i$ and $z_j$ be random vectors, each with support on the vertices of
a probability simplex.
Let $\Sig_{i,j} := \E[z_iz_j^\top]$ and let $\hat\Sig_{i,j}$ be the empirical
average of $N$ independent copies of $z_iz_j^\top$.
Pick any $t > 0$.
With probability at least $1-e^{-t}$,
\[
\Bigl\| \hat\Sig_{i,j} - \Sig_{i,j} \Bigr\|
\leq \Bigl\| \hat\Sig_{i,j} - \Sig_{i,j} \Bigr\|_F
\leq \frac{1 + \sqrt{t}}{\sqrt{N}}
\]
(where $\|A\|_F$ denotes the Frobenius norm of a matrix $A$).
\end{lemma}
For simplicity, we only work with Lemma~\ref{lemma:spectral-bernstein},
although it is easy to translate all of our results by changing the tail
inequality.
The proof of Lemma~\ref{lem:confinterval} is immediate from combining
Lemma~\ref{lemma:spectral-bernstein} and Weyl's Theorem.

Lemma~\ref{lem:confinterval} provides some guidelines on how to set the
$\Delta_{i,j}$ as functions of $N$, $\delta$, and properties of $z_i$ and
$z_j$.
The dependence on the properties of $z_i$ and $z_j$ comes through the
quantities $M_i$, $M_j$, $\bar{d}_{i,j}$, and
\[ B_{i,j} := \max_{i, j} \{ \bigl\|\E[\|z_j\|^2z_iz_i^\top]\bigr\|,
\bigl\|\E[\|z_i\|^2z_jz_j^\top]\bigr\| \} . \]
In practice, one may use plug-in estimates for these quantities, or use
loose upper bounds based on weaker knowledge of the distribution.
For instance, $\bar{d}_{i,j}$ is at most $\max\{ \dim(z_i), \dim(z_j) \}$,
the larger of the explicit vector dimensions of $z_i$ and $z_j$.
Also, if the maximum directional standard deviation $\sigma_*$ of any $z_i$
is known, then $B_{i,j} \leq \max\{ M_i^2, M_j^2 \} \sigma_*^2$.
We note that as these are additive confidence intervals, some dependence on
the properties of $z_i$ and $z_j$ is inevitable.

\section{Analysis of the spectral quartet test}
\label{appendix:section3}

For any hidden variable $h \in \Vhid$, let $\Des(h) \subseteq \Vars$ be the
descendants of $h$ in $\Tree$.
For any $g \in \Des(h) \cap \Vhid$ such that the (directed) path from $h$
to $g$ is $h \to g_1 \to g_2 \to \dotsb \to g_q = g$, define
$\A{g|h} \in \R^{k \times k}$ to be the product
\[ \A{g|h} := \A{g_q|g_{q-1}} \dotsb \A{g_2|g_1} \A{g_1|h} . \]
Similarly, for any $x \in \Des(h) \cap \Vobs$ such that the (directed) path
from $h$ to $x$ is $h \to g_1 \to g_2 \to \dotsb \to g_q \to x$, define
$\C{x|h} \in \R^{d \times k}$ to be the product 
\[ \C{x|h} := \C{x|g_q} \A{g_q|g_{q-1}} \dotsb \A{g_2|g_1} \A{g_1|h} . \]

\subsection{$\log\Det_k$ metric}

Define the function $\mu \colon \Vars \times \Vars \to \R$ by
\[
\mu(u,v) :=
\begin{cases}
\log\Det_k(\E[uu^\top]^{-1/2} \E[uv^\top] \E[vv^\top]^{-1/2})
& \text{if $u,v \in \Vhid$} \\
\log\Det_k(\E[uv^\top] \E[vv^\top]^{-1/2})
& \text{if $u \in \Vobs$, $v \in \Vhid$} \\
\log\Det_k(\E[uu^\top]^{-1/2} \E[uv^\top])
& \text{if $u \in \Vhid$, $v \in \Vobs$} \\
\log\Det_k(\E[uv^\top])
& \text{if $u,v \in \Vobs$}
\end{cases}
.
\]
\begin{proposition}[$\log\Det_k$ metric] \label{proposition:metric}
Assume Conditions~\ref{cond:linear} and~\ref{cond:full-rank} hold, and pick
any $u, v \in \Vars$.
If $w \in \Vars \setminus \{u,v\}$ is on the (undirected) path $u \leadsto
v$, then $\mu(u,v) = \mu(u,w) + \mu(w,v)$.
\end{proposition}
\begin{proof}
Suppose the induced topology over $u,v,w$ in $\Tree$ is the following.
\begin{center}
\begin{tikzpicture}
  [
    scale=1.0,
    observed/.style={circle,minimum size=0.5cm,inner sep=0mm,draw=black,fill=black!20},
    hidden/.style={circle,minimum size=0.5cm,inner sep=0mm,draw=black}
  ]
  \node [hidden,name=u] at ($(-1,0)$) {$u$};
  \node [hidden,name=w] at ($(0,0)$) {$w$};
  \node [hidden,name=v] at ($(1,0)$) {$v$};
  \draw [->] (u) to (w);
  \draw [->] (w) to (v);
\end{tikzpicture}
\end{center}
Assume for now that $u,v \in \Vhid$.
Then, using Condition~\ref{cond:linear},
\[
\E[uv^\top]
= \E[uw^\top] \A{v|w}^\top
= (\E[uw^\top] \E[ww^\top]^{-1/2}) (\E[ww^\top]^{-1/2} \E[wv^\top])
\]
so, because $\rank(\E[uu^\top]^{-1/2} \E[uw^\top] \E[ww^\top]^{-1/2})
= \rank(\E[ww^\top]^{-1/2} \E[wv^\top] \E[vv^\top]^{-1/2}) = k$ by
Condition~\ref{cond:full-rank},
\begin{align*}
\mu(u,v)
& = \log\Det_k(\E[uu^\top]^{-1/2} \E[uw^\top] \E[ww^\top]^{-1/2}
\E[ww^\top]^{-1/2} \E[wv^\top] \E[vv^\top]^{-1/2}) \\
& = \log\Det_k(\E[uu^\top]^{-1/2} \E[uw^\top] \E[ww^\top]^{-1/2})
+ \log\Det_k(\E[ww^\top]^{-1/2} \E[wv^\top] \E[vv^\top]^{-1/2}) \\
& = \mu(u,w) + \mu(w,v)
.
\end{align*}
If $u \in \Vhid$ but $v \in \Vobs$, then let $U_v \in \R^{d \times k}$ be a
matrix of orthonormal left singular vectors of $\C{v|w}$.
Then $\E[uv^\top] = (\E[uw^\top] \E[ww^\top]^{-1/2}) (\E[ww^\top]^{-1/2}
\E[wv^\top])$ as before, and
\begin{align*}
\Det_k(\E[uu^\top]^{-1/2} \E[uv^\top])
& = |\det(\E[uu^\top]^{-1/2} \E[uv^\top] U_v)| \\
& = |\det(\E[uu^\top]^{-1/2})| \cdot |\det(\E[uv^\top] U_v)| \\
& = \Det_k(\E[uu^\top]^{-1/2} \E[uw^\top] \E[ww^\top]^{-1/2})
\cdot \Det_k(\E[ww^\top]^{-1/2} \E[wv^\top] U_v) \\
& = \Det_k(\E[uu^\top]^{-1/2} \E[uw^\top] \E[ww^\top]^{-1/2})
\cdot \Det_k(\E[ww^\top]^{-1/2} \E[wv^\top])
,
\end{align*}
so
\[
\mu(u,v)
= \log\Det_k(\E[uu^\top]^{-1/2} \E[uw^\top] \E[ww^\top]^{-1/2})
+ \log\Det_k(\E[ww^\top]^{-1/2} \E[wv^\top])
= \mu(u,w) + \mu(w,v)
.
\]

Suppose now that the induced toplogy over $u,v,w$ in $\Tree$ is the
following.
\begin{center}
\begin{tikzpicture}
  [
    scale=1.0,
    observed/.style={circle,minimum size=0.5cm,inner sep=0mm,draw=black,fill=black!20},
    hidden/.style={circle,minimum size=0.5cm,inner sep=0mm,draw=black}
  ]
  \node [hidden,name=u] at ($(-1,0)$) {$u$};
  \node [hidden,name=w] at ($(0,0)$) {$w$};
  \node [hidden,name=v] at ($(1,0)$) {$v$};
  \draw [<-] (u) to (w);
  \draw [->] (w) to (v);
\end{tikzpicture}
\end{center}
Again, first assume that $u,v \in \Vhid$.
Then, by Condition~\ref{cond:linear},
\[
\E[uv^\top]
= \A{u|w} \E[ww^\top] \A{v|w}^\top
= (\E[uw^\top] \E[ww^\top]^{-1/2})
(\E[ww^\top]^{-1/2} \E[wv^\top])
,
\]
so $\mu(u,v) = \mu(u,w) + \mu(v,w)$ as before.
The cases where one or both of $u$ and $v$ is in $\Vobs$ follow by similar
arguments as above.
\end{proof}

\subsection{Proof of Lemma~\ref{lemma:spectral-properties}}

By Proposition~\ref{proposition:metric},
\begin{align*}
\Det_k(\E[z_1z_3^\top]) \cdot \Det_k(\E[z_2z_4^\top])
& = \exp(\mu(z_1,z_3) + \mu(z_2,z_4)) \\
& = \exp(\mu(z_1,h) + \mu(h,g) + \mu(g,z_3)
+ \mu(z_2,h) + \mu(h,g) + \mu(g,z_4)) \\
& = \exp(\mu(z_1,h) + \mu(h,g) + \mu(g,z_4)
+ \mu(z_2,h) + \mu(h,g) + \mu(g,z_3)) \\
& = \exp(\mu(z_1,z_4) + \mu(z_2,z_3)) \\
& = \Det_k(\E[z_1z_4^\top]) \cdot \Det_k(\E[z_2z_3^\top])
.
\end{align*}
Moreover,
\begin{align*}
\frac{\Det_k(\E[z_1z_3^\top]) \cdot \Det_k(\E[z_2z_4^\top])}
{\Det_k(\E[z_1z_2^\top]) \cdot \Det_k(\E[z_3z_4^\top])}
& = \frac{\exp(\mu(z_1,z_3) + \mu(z_2,z_4))}
{\exp(\mu(z_1,z_2) + \mu(z_3,z_4))} \\
& = \frac{\exp(\mu(z_1,h) + \mu(h,g) + \mu(g,z_3)
+ \mu(z_2,h) + \mu(h,g) + \mu(g,z_4))}
{\exp(\mu(z_1,h) + \mu(h,z_2) + \mu(z_3,g) + \mu(g,z_4))} \\
& = \exp(2\mu(h,g)) \\
& = \det(\E[hh^\top]^{-1/2} \E[hg^\top] \E[gg^\top]^{-1/2})^2 \\
& = \frac{\det(\E[hg^\top])^2}
{\det(\E[hh^\top]) \cdot \det(\E[gg^\top])}
.
\end{align*}
Finally, note that $u^\top \E[hh^\top]^{-1/2} \E[hg^\top]
\E[gg^\top]^{-1/2} v \leq \|u\| \|v\|$ for all vectors $u$ and $v$ by
Cauchy-Schwarz, so
\[ \frac{\det(\E[hg^\top])^2}{\det(\E[hh^\top]) \cdot \det(\E[gg^\top])}
= \det(\E[hh^\top]^{-1/2} \E[hg^\top] \E[gg^\top]^{-1/2})^2 \leq 1
\]
as required.
\qed

Note that if Condition~\ref{cond:non-redundancy} also holds, then
Lemma~\ref{lemma:spectral-properties} implies the strict inequalities
\[\max\left\{
\Det_k(\E[z_1z_3^\top]) \cdot \Det_k(\E[z_2z_4^\top]), \
\Det_k(\E[z_1z_4^\top]) \cdot \Det_k(\E[z_2z_3^\top])
\right\} < \Det_k(\E[z_1z_2^\top]) \cdot \Det_k(\E[z_3z_4^\top])
. \]

\subsection{Proof of Lemma~\ref{lemma:reliability}}

Given that \eqref{eq:confidence-bounds-test} holds for all pairs
$\{i,j\}$ and all $s \in \{1,2,\dotsc,k\}$, if the spectral quartet test
returns a pairing $\{\{z_i,z_j\},\{z_{i'},z_{j'}\}\}$, it must be that
\begin{align*}
\prod_{s=1}^k \sigma_s(\E[z_iz_j^\top]) \sigma_s(\E[z_{i'}z_{j'}^\top])
& \geq
\prod_{s=1}^k [\sigma_s(\hat\Sig_{i,j}) - \Delta_{i,j}]_+
[\sigma_s(\hat\Sig_{i',j'}) - \Delta_{i',j'}]_+
\\
& >
\prod_{s=1}^k (\sigma_s(\hat\Sig_{i',j}) + \Delta_{i',j})
(\sigma_s(\hat\Sig_{i,j'}) + \Delta_{i,j'})
\geq
\prod_{s=1}^k \sigma_s(\E[z_{i'}z_j^\top]) \sigma_s(\E[z_iz_{j'}^\top])
.
\end{align*}
Therefore
\begin{align*}
\Det_k(\E[z_iz_j^\top]) \cdot \Det_k(\E[z_{i'}z_{j'}^\top])
& = \prod_{s=1}^k \sigma_s(\E[z_iz_j^\top]) \sigma_s(\E[z_{i'}z_{j'}^\top])
\\
& >
\prod_{s=1}^k \sigma_s(\E[z_{i'}z_j^\top]) \sigma_s(\E[z_iz_{j'}^\top])
=
\Det_k(\E[z_{i'}z_j^\top]) \cdot \Det_k(\E[z_iz_{j'}^\top])
.
\end{align*}
But by Lemma~\ref{lemma:spectral-properties}, the above inequality can only
hold if $\{\{z_i,z_j\},\{z_{i'},z_{j'}\}\} = \{\{z_1,z_2\},\{z_3,z_4\}\}$.
\qed

\subsection{Proof of Lemma~\ref{lem:usefulrankk}}

Let $\Sig_{i,j} := \E[z_iz_j^\top]$.
%
The assumptions in the statement of the lemma imply
\[
\max\{ \Delta_{1,2}, \Delta_{3,4} \}
< \frac{\epsilon_0}{8k} \min \{\sigma_k(\Sig_{1,2}),
\sigma_k(\Sig_{3,4})\}
\]
where $\epsilon_0 := \min\left\{ \frac{1}{\rho} - 1, \ \ 1 \right\}$.
Therefore
\begin{align}
\prod_{s=1}^k
[\sigma_s(\hat\Sig_{1,2}) - \Delta_{1,2}]_+
[\sigma_s(\hat\Sig_{3,4}) - \Delta_{3,4}]_+
& \geq \prod_{s=1}^k
[\sigma_s(\Sig_{1,2}) - 2\Delta_{1,2}]_+
[\sigma_s(\Sig_{3,4}) - 2\Delta_{3,4}]_+
\nonumber \\
& >
\left( \prod_{s=1}^k
\sigma_s(\Sig_{1,2})
\sigma_s(\Sig_{3,4}) \right)
\left( 1 - \frac{\epsilon_0}{4k} \right)^{2k}
\nonumber \\
& \geq
\left( \prod_{s=1}^k
\sigma_s(\Sig_{1,2})
\sigma_s(\Sig_{3,4}) \right)
(1 - \epsilon_0/2)
.
\label{eq:det-lb}
\end{align}
If $\E[hg^\top]$ has rank $k$, then so do $\Sig_{i,j}$ for $i \in \{1,2\}$
and $j \in \{3,4\}$.
Therefore, for $\{i',j'\} = \{1,2,3,4\}\setminus\{i,j\}$,
\[
\max\{ \Delta_{i,j}, \Delta_{i',j'} \}
< \frac{\epsilon_0}{8k} \min \{\sigma_k(\Sig_{i',j'}),
\sigma_k(\Sig_{i',j'})\}
.
\]
This implies
\begin{align}
\prod_{s=1}^k
(\sigma_s(\hat\Sig_{i,j}) + \Delta_{i,j})
(\sigma_s(\hat\Sig_{i',j'}) + \Delta_{i',j'})
& \leq
\prod_{s=1}^k
(\sigma_s(\Sig_{i,j}) + 2\Delta_{i,j})
(\sigma_s(\Sig_{i',j'}) + 2\Delta_{i',j'})
\nonumber \\
& <
\left( \prod_{s=1}^k
\sigma_s(\Sig_{i,j})
\sigma_s(\Sig_{i',j'}) \right)
\left( 1 + \frac{\epsilon_0}{4k} \right)^{2k}
\nonumber \\
& \leq
\left( \prod_{s=1}^k
\sigma_s(\Sig_{i,j})
\sigma_s(\Sig_{i',j'}) \right)
(1 + \epsilon_0)
.
\label{eq:det-ub-rank-k}
\end{align}
Therefore, combining~\eqref{eq:det-lb}, \eqref{eq:det-ub-rank-k}, and
Lemma~\ref{lemma:spectral-properties},
\begin{align*}
\lefteqn{
\prod_{s=1}^k
[\sigma_s(\hat\Sig_{1,2}) - \Delta_{1,2}]_+
[\sigma_s(\hat\Sig_{3,4}) - \Delta_{3,4}]_+
} \\
& >
\frac{1 - \epsilon_0/2}{1+\epsilon_0}
\cdot \frac{\det(\E[hh^\top])\det(\E[gg^\top])}{\det(\E[hg^\top])^2}
\cdot
\prod_{s=1}^k
(\sigma_s(\hat\Sig_{i,j}) + \Delta_{i,j})
(\sigma_s(\hat\Sig_{i',j'}) + \Delta_{i',j'})
\\
& \geq
\frac{1}{(1+\epsilon_0)^2}
\cdot \frac{\det(\E[hh^\top])\det(\E[gg^\top])}{\det(\E[hg^\top])^2}
\cdot
\prod_{s=1}^k
(\sigma_s(\hat\Sig_{i,j}) + \Delta_{i,j})
(\sigma_s(\hat\Sig_{i',j'}) + \Delta_{i',j'})
\\
& \geq
\prod_{s=1}^k
(\sigma_s(\hat\Sig_{i,j}) + \Delta_{i,j})
(\sigma_s(\hat\Sig_{i',j'}) + \Delta_{i',j'})
,
\end{align*}
so the spectral quartet test will return the correct pairing
$\{\{z_1,z_2\},\{z_3,z_4\}\}$, proving the lemma.
\qed

\subsection{Conditions for returning a correct pairing when
$\rank(\E[hg^\top]) < k$}

The spectral quartet test is also useful in the case where $\E[hg^\top]$
has rank $r < k$.
In this case, the widths of the confidence intervals are allowed to be
wider than in the case where $\rank(\E[hg^\top]) = k$.
Define
\[ \sigma_{\min}
  := \min \Bigl( \left\{ \sigma_k(\Sig_{1,2}),
  \sigma_k(\Sig_{3,4}) \right\} \cup
  \left\{ \sigma_r(\Sig_{i,j}) \colon i \in
  \{1,2\}, j \in \{3,4\} \right\} \Bigr) . \]
\[ \rho_1^2 = \frac{\sigma_{\min}^{2(k-r)} \cdot \max_{i,j,i',j'} \prod_{s=1}^r
  \sigma_s(\Sig_{i,j}) \sigma_s(\Sig_{i',j'})}
  {\prod_{s=1}^k \sigma_s(\Sig_{1,2}) \sigma_s(\Sig_{3,4})}
  .
\]
Instead of depending on $\min_{i,j} \{ \sigma_k(\Sig_{i,j}) \}$ and $\rho$
as in the case where $\rank(\E[hg^\top]) = k$, we only depend on
$\sigma_{\min}$ and $\rho_1$.

\begin{lemma}[Correct pairing, rank $r < k$]\label{lem:usefulrankr}
Suppose that (i) the observed variables $\Z = \{ z_1, z_2, z_3, z_4 \}$
have the true induced (undirected) topology shown in
Figure~\ref{fig:topologies}(a), (ii) the tree model satisfies
Condition~\ref{cond:linear} and Condition~\ref{cond:full-rank}, (iii)
$\E[hg^\top]$ has rank $r < k$, and (iv) the confidence bounds
in~\eqref{eq:confidence-bounds-test} hold for all $\{i,j\}$ and all $s \in
[k]$.
If
\[ \Delta_{i,j} < \frac{1}{8k} \cdot \min\left\{ 1, \ 8k
\left(\frac{1}{2\rho_1}\right)^{\frac{1}{k - r}} \right\} \cdot
\sigma_{\min} \]
for each $\{i,j\}$, then Algorithm~\ref{alg:quartettest} returns the
correct pairing $\{ \{z_1, z_2\}, \{z_3, z_4\}\}$.
\end{lemma}
Note that the allowed width increases (to a point) as the rank $r$
decreases.

\begin{proof}

The assumptions in the statement of the lemma imply
\[
\max\{ \Delta_{i,j} : \{i,j\} \subset [4] \}
< \frac{\epsilon_1\sigma_{\min}}{8k}
\]
where
\[ \epsilon_1 :=
\min\left\{ 8k \cdot \left(\frac{1}{2\rho_1}\right)^{\frac{1}{k-r}}, \ 1
\right\}
. \]
We have
\[
\prod_{s=1}^k
[\sigma_s(\hat\Sig_{1,2}) - \Delta_{1,2}]_+
[\sigma_s(\hat\Sig_{3,4}) - \Delta_{3,4}]_+
>
\left( \prod_{s=1}^k
\sigma_s(\Sig_{1,2})
\sigma_s(\Sig_{3,4}) \right)
(1 - \epsilon_1/2)
\]
as in the proof of Lemma~\ref{lem:usefulrankk}.
Moreover,
\begin{align*}
\lefteqn{
\prod_{s=1}^k
(\sigma_s(\hat\Sig_{i,j}) + \Delta_{i,j})
(\sigma_s(\hat\Sig_{i',j'}) + \Delta_{i',j'})
} \\
& <
\left( \prod_{s=1}^r \sigma_s(\Sig_{i,j}) \sigma_s(\Sig_{i',j'}) \right)
\cdot (1 + \epsilon_1)
\cdot \left(\frac{\epsilon_1\sigma_{\min}}{8k}\right)^{2(k-r)}
\\
& \leq
\left( \prod_{s=1}^k \sigma_s(\Sig_{1,2}) \sigma_s(\Sig_{3,4}) \right)
\cdot \frac{\rho_1^2}{(\sigma_{\min})^{2(k-r)}}
\cdot (1 + \epsilon_1)
\cdot \left(\frac{\epsilon_1\sigma_{\min}}{8k}\right)^{2(k-r)}
\\
& =
\left( \prod_{s=1}^k \sigma_s(\Sig_{1,2}) \sigma_s(\Sig_{3,4}) \right)
\cdot \rho_1^2
\cdot (1 + \epsilon_1)
\cdot \left(\frac{\epsilon_1}{8k}\right)^{2(k-r)}
\\
& <
\left( \prod_{s=1}^k
[\sigma_s(\hat\Sig_{1,2}) - \Delta_{1,2}]_+
[\sigma_s(\hat\Sig_{3,4}) - \Delta_{3,4}]_+
\right)
\cdot \rho_1^2
\cdot \frac{1 + \epsilon_1}{1-\epsilon_1/2}
\cdot \left(\frac{\epsilon_1}{8k}\right)^{2(k-r)}
\\
& \leq
\left( \prod_{s=1}^k
[\sigma_s(\hat\Sig_{1,2}) - \Delta_{1,2}]_+
[\sigma_s(\hat\Sig_{3,4}) - \Delta_{3,4}]_+
\right)
\cdot \rho_1^2
\cdot (1 + \epsilon_1)^2
\cdot \left(\frac{\epsilon_1}{8k}\right)^{2(k-r)}
\\
& \leq
\prod_{s=1}^k
[\sigma_s(\hat\Sig_{1,2}) - \Delta_{1,2}]_+
[\sigma_s(\hat\Sig_{3,4}) - \Delta_{3,4}]_+
.
\end{align*}
Therefore the spectral quartet test will return the correct pairing
$\{\{z_1,z_2\},\{z_3,z_4\}\}$; the lemma follows.
\end{proof}

\section{Analysis of Spectral Recursive Grouping}

\subsection{Overview}

Here is an outline of the argument for Theorem~\ref{theorem:rg}.
\begin{enumerate}
\item First, we condition on a $1-\eta$ probability event over the iid
samples from the distribution over $\Vobs$ in which the empirical
second-moment matrices are sufficiently close to the true second-moment
matrices in by spectral norm (Equation~\ref{eq:confidence-bounds-event}).
This is required to reason deterministically about the behavior of the
algorithm.

\item Next, we characterize the pairs $\{u,v\} \subseteq \Roots$ (where
$\Roots$ are the roots of subtrees maintained by the algorithm) that cause
the $\Mergeable$ subroutine to return true.
(Lemma~\ref{lemma:mergeable}), as well as those that cause it to return
false (Lemma~\ref{lemma:not-mergeable}).

\item We use the above characterizations to show that the main while-loop
of the algorithm maintains loop invariants such that when the loop finally
terminates, the entire tree structure will have been completely discovered
(Lemma~\ref{lemma:loop-invariant}).
This is achieved by showing each iteration of the while-loop
\begin{enumerate}

\item selects a ``$\Mergeable$'' pair $\{u,v\} \subseteq \Roots$ that
satisfies certain properties (Claim~\ref{claim:goodpair} and
Claim~\ref{claim:badpair}) such that, if they are properly combined (as
siblings or parent/child), the required loop invariants will be perserved;
and

\item uses the $\Relationship$ subroutine to correctly determine whether
the chosen pair $\{u,v\}$ should be combined as siblings or parent/child
(Claim~\ref{claim:relationship}).

\end{enumerate}

\end{enumerate}

\subsection{Proof of Theorem~\ref{theorem:rg}}

Recall the definitions of $\A{g|h} \in \R^{k \times k}$ and $\C{x|h} \in
\R^{d \times k}$ for descendants $g \in \Des(h) \cap \Vhid$ and $x \in
\Des(h) \cap \C{x|h}$ in $\Tree$, as given in
Appendix~\ref{appendix:section3}.

Let us define
\begin{align*}
\epsmin & := \min\left\{ \frac1{\rhomax} - 1, \ 1 \right\} ,
& \veps & := \frac{\gammamin/\gammamax}{8k+\gammamin/\gammamax} ,
\\
\theta & := \frac{\gammamin}{1 + \veps} ,
& \vsig & := \frac{\gammamin}{\gammamax} \cdot
(1-\veps) \cdot \theta
.
\end{align*}
The sample size requirement ensures that
\[
\Delta_{x_i,x_j} <
\frac{\epsmin \cdot \vsig}{8k}
\leq \veps \theta
.
\]
This implies conditions on the thresholds $\Delta_{x_i,x_j}$ in
Lemma~\ref{lem:usefulrankk} for the spectral quartet test on
$\{x_1,x_2,x_3,x_4\}$ to return a correct pairing, provided that
\begin{equation} \label{eq:min-correlation-requirement}
\min \{ \sigma_k(\Sig_{x_i,x_j}) : \{i,j\}
\subset \{1,2,3,4\} \}
\geq \varsigma
.
\end{equation}

The probabilistic event we need is that in which the confidence bounds from
Lemma~\ref{lemma:spectral-bernstein} hold for each pair of observed
variables.
The event
\begin{equation} \label{eq:confidence-bounds-event}
\forall \{x_i,x_j\} \subseteq \Vobs \centerdot
\|\wh\Sig_{x_i,x_j} - \Sig_{x_i,x_j}\|
\leq \Delta_{x_i,x_j}
,
\end{equation}
occurs with probability at least $1-\eta$ by
Lemma~\ref{lemma:spectral-bernstein} and a union bound.
We henceforth condition on the above event.

The following is an immediate consequence of Weyl's Theorem and
conditioning on the above event.
\begin{lemma} \label{lemma:weyl-application}
Fix any pair $\{x,y\} \subseteq \Vobs$.
If $\sigma_k(\Sig_{x,y}) \geq (1+\veps)\theta$, then
$\sigma_k(\wh\Sig_{x,y}) \geq \theta$.
If $\sigma_k(\wh\Sig_{x,y}) \geq \theta$, then $\sigma_k(\Sig_{x,y}) \geq
(1-\veps)\theta$.
\end{lemma}

Before continuing, we need some definitions and notation.
First, we refer to the variables in $\Vars$ interchangeably as both nodes
and variables.
Next, we generally ignore the direction of edges in $\Tree$, except when it
becomes crucial (namely, in Lemma~\ref{lemma:transfer}).
For a node $r$ in $\Tree$, we say that a subtree $\Subtree[r]$ of $\Tree$
(ignoring edge directions) is \emph{rooted at $r$} if $\Subtree[r]$
contains $r$, and for every node $u$ in $\Subtree[r]$ and any node $v$ not
in $\Subtree[r]$, the (undirected) path from $u$ to $v$ in $\Tree$ passes
through $r$.
Note that a rooted subtree naturally imply parent/child relationships
between its constituent nodes, and it is in this sense we use the terms
``parent'', ``child'', ``sibling'', etc.~throughout the analysis, rather
than in the sense given by the edge directions in $\Tree$ (the exception is
in Lemma~\ref{lemma:transfer}).
A collection $\Comps$ of disjoint rooted subtrees of $\Tree$ naturally
gives rise to a \emph{super-tree} $\Supertree[\Comps]$ by starting with
$\Tree$ and then collapsing each $\Subtree[r] \in \Comps$ into a single
node.
Note that each node in $\Supertree[\Comps]$ is either associated with a
subtree in $\Comps$, or is a node in $\Tree$ that doesn't appear in any
subtree in $\Comps$.
We say a subtree $\Subtree \in \Comps$ is a \emph{leaf component relative
to $\Comps$} if it is a leaf in this super-tree $\Supertree[\Comps]$.
Finally, define $\Vhid[\Comps] := \{ h \in \Vhid : \text{$h$ does not
appear in any subtree in $\Comps$} \}$.

The following lemma is a simple fact about the super-tree given properties
on the subtrees (which will be maintained by the algorithm).
\begin{lemma}[Super-tree property] \label{lemma:supertree}
Let $\Roots \subseteq \Vars$.
Let $\Comps := \{ \Subtree[u] : u \in \Roots \}$ be a collection of
disjoint rooted subtrees, with $u$ being the root of $\Subtree[u]$, such
that their leaf sets $\{ \Leaves[u] : u \in \Roots \}$ partition $\Vobs$.
Then the nodes of the super-tree $\Supertree[\Comps]$ are $\Comps \cup
\Vhid[\Comps]$, and the leaves of $\Supertree[\Comps]$ are all in $\Comps$.
\end{lemma}
\begin{proof}
This follows because each leaf in $\Tree$ appears in the leaf set of some
$\Subtree[u]$.
\end{proof}

The next lemma relates the correlation between two observed variables in a
quartet (on opposite sides of the bottleneck) to the correlations of the
other pairs crossing the bottleneck.
\begin{lemma}[Correlation transfer] \label{lemma:transfer}
Consider the following induced (undirected) topology over $\{z_1, z_2, z_3,
z_4 \} \subseteq \Vobs$.
\begin{center}
\begin{tikzpicture}
  [
    scale=1.0,
    observed/.style={circle,minimum size=0.5cm,inner sep=0mm,draw=black,fill=black!20},
    hidden/.style={circle,minimum size=0.5cm,inner sep=0mm,draw=black}
  ]
  \node [observed,name=z1] at ($(-1,0.5)$) {$z_1$};
  \node [observed,name=z2] at ($(-1,-0.5)$) {$z_2$};
  \node [observed,name=z3] at ($(1,0.5)$) {$z_3$};
  \node [observed,name=z4] at ($(1,-0.5)$) {$z_4$};
  \node [hidden,name=h] at ($(-1/3,0)$) {$h$};
  \node [hidden,name=g] at ($(1/3,0)$) {$g$};
  \draw [-] (h) to (z1);
  \draw [-] (h) to (z2);
  \draw [-] (h) to (g);
  \draw [-] (g) to (z3);
  \draw [-] (g) to (z4);
\end{tikzpicture}
\end{center}
Then
\[
\sigma_k(\E[z_1z_4^\top])
\geq \frac{\sigma_k(\E[z_1z_3^\top])
\sigma_k(\E[z_2z_4^\top])}{\sigma_1(\E[z_2z_3^\top])}
.
\]
\end{lemma}
\begin{proof}
In this proof, the edge directions and the notion of ancestor are
determined according to the edge directions in $\Tree$.
Let $r$ be the least common ancestor of $\{z_1,z_2,z_3,z_4\}$ in $\Tree$.
There are effectively three possible cases to consider, depending on the
location of $r$ relative to the $z_i$, $h$, and $g$; we may exploit the
fact that $\sigma_k(\E[z_1z_4^\top]) = \sigma_k(\E[z_4z_1^\top])$ to cover
the remaining cases.
\begin{enumerate}
\item Suppose $r$ appears between $h$ and $z_1$.
\begin{center}
\begin{tikzpicture}
  [
    scale=1.0,
    observed/.style={circle,minimum size=0.5cm,inner sep=0mm,draw=black,fill=black!20},
    hidden/.style={circle,minimum size=0.5cm,inner sep=0mm,draw=black}
  ]
  \node [observed,name=z1] at ($(-2,1)$) {$z_1$};
  \node [observed,name=z2] at ($(-2,-1)$) {$z_2$};
  \node [observed,name=z3] at ($(2,1)$) {$z_3$};
  \node [observed,name=z4] at ($(2,-1)$) {$z_4$};
  \node [hidden,name=r] at ($(-4/3,0.5)$) {$r$};
  \node [hidden,name=h] at ($(-2/3,0)$) {$h$};
  \node [hidden,name=g] at ($(2/3,0)$) {$g$};
  \draw [->] (r) to (z1);
  \draw [->] (r) to (h);
  \draw [->] (h) to (z2);
  \draw [->] (h) to (g);
  \draw [->] (g) to (z3);
  \draw [->] (g) to (z4);
\end{tikzpicture}
\end{center}
By Condition~\ref{cond:full-rank}, we can choose matrices $U_1, U_2, U_3,
U_4 \in \R^{d \times k}$ such that
the columns of $U_1$ are an orthonormal basis of $\range(\C{z_1|r})$,
the columns of $U_2$ are an orthonormal basis of $\range(\C{z_2|h})$,
the columns of $U_3$ are an orthonormal basis of $\range(\C{z_3|g})$,
and the columns of $U_4$ are an orthonormal basis of $\range(\C{z_4|g})$.
We have
\begin{align*}
\lefteqn{
U_1^\top \E[z_1z_4^\top] U_4
} \\
& = U_1^\top \C{z_1|r} \E[rr^\top] \A{h|r}^\top \C{z_4|h}^\top U_4 \\
& =
(U_1^\top \C{z_1|r} \E[rr^\top]) \A{h|r}^\top
(\C{z_3|h}^\top U_3)
(\C{z_3|h}^\top U_3)^{-1}
\\
& \qquad{}
(U_2^\top \C{z_2|h} \E[hh^\top])^{-1}
(U_2^\top \C{z_2|h} \E[hh^\top])
(\C{z_4|h}^\top U_4) \\
& =
(U_1^\top \C{z_1|r} \E[rr^\top] \A{h|r}^\top \C{z_3|h}^\top U_3)
(U_2^\top \C{z_2|h} \E[hh^\top] \C{z_3|h}^\top U_3)^{-1}
\\
& \qquad{}
(U_2^\top \C{z_2|h} \E[hh^\top] \C{z_4|h}^\top U_4) \\
& =
(U_1^\top \E[z_1z_3^\top] U_3)
(U_2^\top \E[z_2z_3^\top] U_3)^{-1}
(U_2^\top \E[z_2z_4^\top] U_4)
.
\end{align*}

\item Suppose $r$ appears between $h$ and $z_2$.
\begin{center}
\begin{tikzpicture}
  [
    scale=1.0,
    observed/.style={circle,minimum size=0.5cm,inner sep=0mm,draw=black,fill=black!20},
    hidden/.style={circle,minimum size=0.5cm,inner sep=0mm,draw=black}
  ]
  \node [observed,name=z1] at ($(-2,1)$) {$z_1$};
  \node [observed,name=z2] at ($(-2,-1)$) {$z_2$};
  \node [observed,name=z3] at ($(2,1)$) {$z_3$};
  \node [observed,name=z4] at ($(2,-1)$) {$z_4$};
  \node [hidden,name=r] at ($(-4/3,-0.5)$) {$r$};
  \node [hidden,name=h] at ($(-2/3,0)$) {$h$};
  \node [hidden,name=g] at ($(2/3,0)$) {$g$};
  \draw [->] (r) to (z2);
  \draw [->] (r) to (h);
  \draw [->] (h) to (z1);
  \draw [->] (h) to (g);
  \draw [->] (g) to (z3);
  \draw [->] (g) to (z4);
\end{tikzpicture}
\end{center}
By Condition~\ref{cond:full-rank}, we can choose matrices $U_1, U_2, U_3,
U_4 \in \R^{d \times k}$ such that
the columns of $U_1$ are an orthonormal basis of $\range(\C{z_1|h})$,
the columns of $U_2$ are an orthonormal basis of $\range(\C{z_2|r})$,
the columns of $U_3$ are an orthonormal basis of $\range(\C{z_3|g})$,
and the columns of $U_4$ are an orthonormal basis of $\range(\C{z_4|g})$.
We have
\begin{align*}
\lefteqn{
U_1^\top \E[z_1z_4^\top] U_4
} \\
& = U_1^\top \C{z_1|h} \E[hh^\top] \A{h|r}^{-\top} \C{z_4|r}^\top U_4 \\
& =
(U_1^\top \C{z_1|h} \E[hh^\top])
(\C{z_3|h}^\top U_3)
(\C{z_3|h}^\top U_3)^{-1}
\A{h|r}^{-\top}
\\
& \qquad{}
(U_2^\top \C{z_2|r} \E[rr^\top])^{-1}
(U_2^\top \C{z_2|r} \E[rr^\top])
(\C{z_4|r}^\top U_4) \\
& =
(U_1^\top \C{z_1|h} \E[hh^\top] \C{z_3|h}^\top U_3)
(U_2^\top \C{z_2|r} \E[rr^\top] \A{h|r}^\top \C{z_3|h}^\top U_3)^{-1}
\\
& \qquad{}
(U_2^\top \C{z_2|r} \E[rr^\top] \C{z_4|r}^\top U_4) \\
& =
(U_1^\top \E[z_1z_3^\top] U_3)
(U_2^\top \E[z_2z_3^\top] U_3)^{-1}
(U_2^\top \E[z_2z_4^\top] U_4)
.
\end{align*}

\item Suppose either $r = h$, or $r$ is between $h$ and $g$.
\begin{center}
\begin{tikzpicture}
  [
    scale=1.0,
    observed/.style={circle,minimum size=0.5cm,inner sep=0mm,draw=black,fill=black!20},
    hidden/.style={circle,minimum size=0.5cm,inner sep=0mm,draw=black}
  ]
  \node [observed,name=z1] at ($(-2,1)$) {$z_1$};
  \node [observed,name=z2] at ($(-2,-1)$) {$z_2$};
  \node [observed,name=z3] at ($(2,1)$) {$z_3$};
  \node [observed,name=z4] at ($(2,-1)$) {$z_4$};
  \node [hidden,name=r] at ($(-2/3,0)$) {$r$};
  \node [hidden,name=g] at ($(2/3,0)$) {$g$};
  \draw [->] (r) to (z1);
  \draw [->] (r) to (z2);
  \draw [->] (r) to (g);
  \draw [->] (g) to (z3);
  \draw [->] (g) to (z4);
\end{tikzpicture}
\qquad
\begin{tikzpicture}
  [
    scale=1.0,
    observed/.style={circle,minimum size=0.5cm,inner sep=0mm,draw=black,fill=black!20},
    hidden/.style={circle,minimum size=0.5cm,inner sep=0mm,draw=black}
  ]
  \node [observed,name=z1] at ($(-2,1)$) {$z_1$};
  \node [observed,name=z2] at ($(-2,-1)$) {$z_2$};
  \node [observed,name=z3] at ($(2,1)$) {$z_3$};
  \node [observed,name=z4] at ($(2,-1)$) {$z_4$};
  \node [hidden,name=r] at ($(0,0)$) {$r$};
  \node [hidden,name=h] at ($(-1,0)$) {$h$};
  \node [hidden,name=g] at ($(1,0)$) {$g$};
  \draw [->] (r) to (h);
  \draw [->] (r) to (g);
  \draw [->] (h) to (z1);
  \draw [->] (h) to (z2);
  \draw [->] (g) to (z3);
  \draw [->] (g) to (z4);
\end{tikzpicture}
\end{center}
In either case, by Condition~\ref{cond:full-rank}, we can choose matrices
$U_1, U_2, U_3, U_4 \in \R^{d \times k}$ such that the columns of $U_1$ are
an orthonormal basis of $\range(\C{z_1|h})$, the columns of $U_2$ are an
orthonormal basis of $\range(\C{z_2|h})$, the columns of $U_3$ are an
orthonormal basis of $\range(\C{z_3|g})$, and the columns of $U_4$ are an
orthonormal basis of $\range(\C{z_4|g})$.
We have
\begin{align*}
\lefteqn{
U_1^\top \E[z_1z_4^\top] U_4
} \\
& = U_1^\top \C{z_1|r} \E[rr^\top] \C{z_4|r}^\top U_4 \\
& =
(U_1^\top \C{z_1|r} \E[rr^\top])
(\C{z_3|r}^\top U_3)
(\C{z_3|r}^\top U_3)^{-1}
\\
& \qquad{}
(U_2^\top \C{z_2|r} \E[rr^\top])^{-1}
(U_2^\top \C{z_2|r} \E[rr^\top])
(\C{z_4|r}^\top U_4) \\
& =
(U_1^\top \C{z_1|r} \E[rr^\top] \C{z_3|r}^\top U_3)
(U_2^\top \C{z_2|r} \E[rr^\top] \C{z_3|r}^\top U_3)^{-1}
\\
& \qquad{}
(U_2^\top \C{z_2|r} \E[rr^\top] \C{z_4|r}^\top U_4) \\
& =
(U_1^\top \E[z_1z_3^\top] U_3)
(U_2^\top \E[z_2z_3^\top] U_3)^{-1}
(U_2^\top \E[z_2z_4^\top] U_4)
.
\end{align*}

\end{enumerate}
Therefore, in all cases,
\[ \sigma_k(\E[z_1z_4^\top]) \geq \frac{\sigma_k(\E[z_1z_3^\top]) \cdot
\sigma_k(\E[z_2z_4^\top])}{\sigma_1(\E[z_2z_3^\top])}
.
\qedhere
\]
\end{proof}

The next two lemmas (Lemmas~\ref{lemma:mergeable}
and~\ref{lemma:not-mergeable}) show a dichotomy in the cases that cause the
subroutine $\Mergeable$ return either true or false.
\sloppy
\begin{lemma}[Mergeable pairs] \label{lemma:mergeable}
Let $\Roots \subseteq \Vars$.
Let $\Comps := \{ \Subtree[r] : r \in \Roots \}$ be a collection of
disjoint rooted subtrees, with $r$ being the root of $\Subtree[r]$, such
that their leaf sets $\{ \Leaves[r] : r \in \Roots \}$ partition $\Vobs$.
Further, suppose the pair $\{u,v\} \subseteq \Roots$ are such that one of
the following conditions hold.
\begin{enumerate}
\item $\{u,v\}$ share a common neighbor in $\Tree$, and both of
$\Subtree[u]$ and $\Subtree[v]$ are leaf components relative to $\Comps$.

\item $\{u,v\}$ are neighbors in $\Tree$, and at least one of $\Subtree[u]$
and $\Subtree[v]$ is a leaf component relative to $\Comps$.

\end{enumerate}
Then for all pairs $\{u_1,v_1\} \subseteq \Roots \setminus \{u,v\}$ and all
$(x,y,x_1,y_1) \in \Leaves[u] \times \Leaves[v] \times \Leaves[u_1] \times
\Leaves[v_1]$, $\SQT(\{x,y,x_1,y_1\})$ returns $\{\{x,y\},\{x_1,y_1\}\}$ or
$\bot$.
This implies that $\Mergeable(\Roots,\Leaves[\cdot],u,v)$ returns true.
\end{lemma}
\fussy
\begin{remark}
Note that if $|\Roots| < 4$, then $\Mergeable(\Roots,\Leaves[\cdot],u,v)$
returns true for all pairs $\{u,v\} \subseteq \Roots$.
\end{remark}
\begin{proof}
Suppose the first condition holds, and let $h$ be the common neighbor.
Since $\Subtree[u]$ is a leaf component relative to $\Comps$, the
(undirected) path from any node $u'$ in $\Subtree[u]$ to another node $w$
not in $\Subtree[u]$ must pass through $h$.
Similarly, the (undirected) path from any node $v'$ in $\Subtree[v]$ to
another node $w$ not in $\Subtree[v]$ must pass through $h$.
Therefore, each choice of $\{u_1,v_1\} \subseteq \Roots \setminus \{u,v\}$
and $(x,y,x_1,y_1) \in \Leaves[u] \times \Leaves[v] \times \Leaves[u_1]
\times \Leaves[v_1]$ induces one of the following topologies,
\begin{center}
\begin{tabular}{cc}
\begin{tikzpicture}
  [
    scale=1.0,
    observed/.style={circle,minimum size=0.5cm,inner sep=0mm,draw=black,fill=black!20},
    hidden/.style={circle,minimum size=0.5cm,inner sep=0mm,draw=black}
  ]
  \node [observed,name=x] at ($(-1,0.5)$) {$x$};
  \node [observed,name=y] at ($(-1,-0.5)$) {$y$};
  \node [observed,name=x1] at ($(1,0.5)$) {$x_1$};
  \node [observed,name=y1] at ($(1,-0.5)$) {$y_1$};
  \node [hidden,name=h] at ($(-1/3,0)$) {$h$};
  \node [hidden,name=g] at ($(1/3,0)$) {};
  \draw [-] (x) to (h);
  \draw [-] (y) to (h);
  \draw [-] (x1) to (g);
  \draw [-] (y1) to (g);
  \draw [-] (h) to (g);
\end{tikzpicture}
&
\begin{tikzpicture}
  [
    scale=1.0,
    observed/.style={circle,minimum size=0.5cm,inner sep=0mm,draw=black,fill=black!20},
    hidden/.style={circle,minimum size=0.5cm,inner sep=0mm,draw=black}
  ]
  \node [observed,name=x] at ($(-1,0.5)$) {$x$};
  \node [observed,name=y] at ($(-1,-0.5)$) {$y$};
  \node [observed,name=x1] at ($(1,0.5)$) {$x_1$};
  \node [observed,name=y1] at ($(1,-0.5)$) {$y_1$};
  \node [hidden,name=h] at ($(0,0)$) {$h$};
  \draw [-] (x) to (h);
  \draw [-] (y) to (h);
  \draw [-] (x1) to (h);
  \draw [-] (y1) to (h);
\end{tikzpicture}
\end{tabular}
\end{center}
upon which, by Lemma~\ref{lemma:reliability}, the quartet test returns
either $\{\{x,y\},\{x_1,y_1\}\}$ or $\bot$.

Now instead suppose the second condition holds.
Without loss of generality, assume $\Subtree[u]$ is a leaf component
relative to $\Comps$, which then implies that the (undirected) path from
any node $u'$ in $\Subtree[u]$ to another node $w$ not in $\Subtree[u]$
must pass through $v$.
Moreover, since $\Subtree[v]$ is rooted at $v$, the (undirected) path from
any node $v'$ in $\Subtree[v]$ to another node $w$ not in $\Subtree[v]$
must pass through $v$.
If $\Subtree[v]$ is also a leaf component, then it must be that $\Roots =
\{u,v\}$, in which case $\Roots \setminus \{u,v\} = \emptyset$.
If $\Subtree[v]$ is not a leaf component, then each choice of $\{u_1,v_1\}
\subseteq \Roots \setminus \{u,v\}$ and $(x,y,x_1,y_1) \in \Leaves[u]
\times \Leaves[v] \times \Leaves[u_1] \times \Leaves[v_1]$ induces one of
the following topologies,
\begin{center}
\begin{tabular}{cc}
\begin{tikzpicture}
  [
    scale=1.0,
    observed/.style={circle,minimum size=0.5cm,inner sep=0mm,draw=black,fill=black!20},
    hidden/.style={circle,minimum size=0.5cm,inner sep=0mm,draw=black}
  ]
  \node [observed,name=x] at ($(-1,0.5)$) {$x$};
  \node [observed,name=y] at ($(-1,-0.5)$) {$y$};
  \node [observed,name=x1] at ($(1,0.5)$) {$x_1$};
  \node [observed,name=y1] at ($(1,-0.5)$) {$y_1$};
  \node [hidden,name=v] at ($(-1/3,0)$) {$v$};
  \node [hidden,name=g] at ($(1/3,0)$) {};
  \draw [-] (x) to (v);
  \draw [-] (y) to (v);
  \draw [-] (x1) to (g);
  \draw [-] (y1) to (g);
  \draw [-] (v) to (g);
\end{tikzpicture}
&
\begin{tikzpicture}
  [
    scale=1.0,
    observed/.style={circle,minimum size=0.5cm,inner sep=0mm,draw=black,fill=black!20},
    hidden/.style={circle,minimum size=0.5cm,inner sep=0mm,draw=black}
  ]
  \node [observed,name=x] at ($(-1,0.5)$) {$x$};
  \node [observed,name=y] at ($(-1,-0.5)$) {$y$};
  \node [observed,name=x1] at ($(1,0.5)$) {$x_1$};
  \node [observed,name=y1] at ($(1,-0.5)$) {$y_1$};
  \node [hidden,name=v] at ($(0,0)$) {$v$};
  \draw [-] (x) to (v);
  \draw [-] (y) to (v);
  \draw [-] (x1) to (h);
  \draw [-] (y1) to (h);
\end{tikzpicture}
\end{tabular}
\end{center}
upon which, by Lemma~\ref{lemma:reliability}, the quartet test returns
either $\{\{x,y\},\{x_1,y_1\}\}$ or $\bot$.
\end{proof}

\sloppy
\begin{lemma}[Un-mergeable pairs] \label{lemma:not-mergeable}
Let $\Roots \subseteq \Vars$.
Let $\Comps := \{ \Subtree[r] : r \in \Roots \}$ be a collection of
disjoint rooted subtrees, with $r$ being the root of $\Subtree[r]$, such
that their leaf sets $\{ \Leaves[r] : r \in \Roots \}$ partition $\Vobs$.
Further, suppose the pair $\{u,v\} \subseteq \Roots$ are such that all of
the following conditions hold.
\begin{enumerate}
\item There exists $(x,y) \in \Leaves[u] \times \Leaves[v]$ such that
$\sigma_k(\wh\Sig_{x,y}) \geq \theta$.

\item $\{u,v\}$ do not share a common neighbor in $\Tree$, or at least one
of $\Subtree[u]$ and $\Subtree[v]$ is not a leaf component relative to
$\Comps$.

\item $\{u,v\}$ are not neighbors in $\Tree$, or neither $\Subtree[u]$ nor
$\Subtree[v]$ is a leaf component relative to $\Comps$.

\end{enumerate}
Then there exists a pair $\{u_1,v_1\} \subseteq \Roots \setminus \{u,v\}$
and $(x_1,y_1) \in \Leaves[u_1] \times \Leaves[v_1]$ such that
$\SQT(\{x,y,x_1,y_1\})$ returns $\{\{x,x_1\},\{y,y_1\}\}$.
This implies that $\Mergeable(\Roots,\Leaves[\cdot],u,v)$ returns false.
\end{lemma}
\fussy
\begin{proof}
First, take $(x,y) \in \Leaves[u] \times \Leaves[v]$ such that
$\sigma_k(\wh\Sig_{x,y}) \geq \theta$.
By Lemma~\ref{lemma:weyl-application}, $\sigma_k(\Sig_{x,y}) \geq
(1-\veps)\theta$.
Lemma~\ref{lemma:supertree} implies that the nodes of $\Supertree[\Comps]$
are $\Comps \cup \Vhid[\Comps]$, and that each leaf in $\Supertree[\Comps]$
is a subtree $\Subtree[u] \in \Comps$.
The second and third conditions of the lemma on $\{u,v\}$ imply that at
least one of the following cases holds.
\begin{enumerate}
\renewcommand{\labelenumi}{(\roman{enumi})}
\item Neither $\Subtree[u]$ nor $\Subtree[v]$ is a leaf component relative
to $\Comps$.

\item $u$ and $v$ are not neighbors and do not share a common neighbor.

\item $u$ and $v$ are not neighbors, and one of $\Subtree[u]$ and
$\Subtree[v]$ is not a leaf component relative to $\Comps$.

\end{enumerate}

Suppose (i) holds.
Then each of $\Subtree[u]$ and $\Subtree[v]$ have degree $\geq2$ in
$\Supertree[\Comps]$.
Note that neither $u$ nor $v$ are leaves in $\Tree$.
Moreover, there exists $\{u_1,v_1\} \subseteq (\Roots \setminus \{u,v\})
\cup \Vhid[\Comps]$ such that $u_1$ is adjacent to $u$ in $\Tree$, $v_1$ is
adjacent to $v$ in $\Tree$, and the (undirected) path from $u_1$ to $v_1$
in $\Tree$ intersects the (undirected) path from $u$ to $v$ in $\Tree$.
\begin{center}
\begin{tikzpicture}
  [
    scale=1.0,
    observed/.style={circle,minimum size=0.5cm,inner sep=0mm,draw=black,fill=black!20},
    hidden/.style={circle,minimum size=0.5cm,inner sep=0mm,draw=black}
  ]
  \node [hidden,name=u] at ($(-1,0)$) {$u$};
  \node [hidden,name=u1] at ($(-2,0)$) {$u_1$};
  \node [hidden,name=v1] at ($(2,0)$) {$v_1$};
  \node [hidden,name=v] at ($(1,0)$) {$v$};
  \draw [-] (u) to (u1);
  \draw [loosely dashed] (u) to (v);
  \draw [-] (v1) to (v);
\end{tikzpicture}
\end{center}
Since $u$ is not a leaf, it has at least three neighbors by assumption, and
thus there exist three subtrees $\{\Subtree_{u,1}, \Subtree_{u,2},
\Subtree_{u,3}\} \subseteq \Forest_u$ such that $u_1$ is the root of
$\Subtree_{u,1}$, $x \in \Vobs[\Subtree_{u,2}]$ and $y \in
\Vobs[\Subtree_{u,3}]$.
Moreover, by Condition~\ref{cond:correlation}, there exist
$x_1 \in \Vobs[\Subtree_{u,1}]$,
$x_2 \in \Vobs[\Subtree_{u,2}]$, and
$x_3 \in \Vobs[\Subtree_{u,3}]$ such that $\sigma_k(\E[x_ix_j^\top]) \geq
\gammamin$ for all $\{i,j\} \subset \{1,2,3\}$.
Note that it is possible to have $x_2 = x$ and $x_3 = y$.
Let $u_2$ denote the node in $\Subtree_{u,2}$ at which the (undirected)
paths $x \leadsto u$ and $x_2 \leadsto u$ intersect (if $x_2 = x$, then let
$u_2$ be the root of $\Subtree_{u,2}$); similarly, let $u_3$ denote the
node in $\Subtree_{u,2}$ at which the (undirected) paths $y \leadsto u$ and
$x_3 \leadsto u$ intersect (if $x_3 = y$, then let $u_3$ be the root of
$\Subtree_{u,3}$).
The induced (undirected) topology over these nodes is shown below.
\begin{center}
\begin{tikzpicture}
  [
    scale=1.0,
    observed/.style={circle,minimum size=0.5cm,inner sep=0mm,draw=black,fill=black!20},
    hidden/.style={circle,minimum size=0.5cm,inner sep=0mm,draw=black},
  ]
  \node [hidden,name=u] at ($(0,0)$) {$u$};
  \node [hidden,name=u1] at ($(-2,-1)$) {$u_1$};
  \node [hidden,name=u2] at ($(0,-1)$) {$u_2$};
  \node [hidden,name=u3] at ($(2,-1)$) {$u_3$};
  \node [observed,name=x] at ($(0.5,-2)$) {$x$};
  \node [observed,name=y] at ($(2.5,-2)$) {$y$};
  \node [observed,name=x1] at ($(-2,-2)$) {$x_1$};
  \node [observed,name=x2] at ($(-0.5,-2)$) {$x_2$};
  \node [observed,name=x3] at ($(1.5,-2)$) {$x_3$};
  \draw [-] (u) to (u1);
  \draw [-] (u) to (u2);
  \draw [-] (u) to (u3);
  \draw [-] (u1) to (x1);
  \draw [-] (u2) to (x2);
  \draw [-] (u2) to (x);
  \draw [-] (u3) to (x3);
  \draw [-] (u3) to (y);
\end{tikzpicture}
\end{center}
A completely analogous argument can be applied relative to $v$ instead of
$u$, giving the following.
\begin{center}
\begin{tikzpicture}
  [
    scale=1.0,
    observed/.style={circle,minimum size=0.5cm,inner sep=0mm,draw=black,fill=black!20},
    hidden/.style={circle,minimum size=0.5cm,inner sep=0mm,draw=black},
  ]
  \node [hidden,name=v] at ($(0,0)$) {$v$};
  \node [hidden,name=v1] at ($(-2,-1)$) {$v_1$};
  \node [hidden,name=v2] at ($(0,-1)$) {$v_2$};
  \node [hidden,name=v3] at ($(2,-1)$) {$v_3$};
  \node [observed,name=y] at ($(0.5,-2)$) {$y$};
  \node [observed,name=x] at ($(2.5,-2)$) {$x$};
  \node [observed,name=y1] at ($(-2,-2)$) {$y_1$};
  \node [observed,name=y2] at ($(-0.5,-2)$) {$y_2$};
  \node [observed,name=y3] at ($(1.5,-2)$) {$y_3$};
  \draw [-] (v) to (v1);
  \draw [-] (v) to (v2);
  \draw [-] (v) to (v3);
  \draw [-] (v1) to (y1);
  \draw [-] (v2) to (y2);
  \draw [-] (v2) to (y);
  \draw [-] (v3) to (y3);
  \draw [-] (v3) to (x);
\end{tikzpicture}
\end{center}
\begin{claim} \label{claim:min-correlation}
The following lower bounds hold.
\begin{equation} \label{eq:min-correlation}
\min\left\{ \sigma_k(\Sig_{x_1,x}),
\ \sigma_k(\Sig_{x_1,y}),
\ \sigma_k(\Sig_{y_1,y}),
\ \sigma_k(\Sig_{y_1,x})
\right\}
\geq \frac{\gammamin \cdot (1-\veps)\theta}{\gammamax}
= \varsigma
.
\end{equation}
\end{claim}
\begin{proof}
We just show the inequalities for $\sigma_k(\E[x_1x^\top])$ and
$\sigma_k(\E[x_1y^\top])$; the other two are analogous.
If $x_2 = x$, then $\sigma_k(\E[x_1x^\top]) = \sigma_k(\E[x_1x_2^\top])
\geq \gammamin \geq \varsigma$.
If $x_2 \neq x$, then we have the following induced (undirected) topology.
\begin{center}
\begin{tikzpicture}
  [
    scale=1.0,
    observed/.style={circle,minimum size=0.5cm,inner sep=0mm,draw=black,fill=black!20},
    hidden/.style={circle,minimum size=0.5cm,inner sep=0mm,draw=black}
  ]
  \node [observed,name=x1] at ($(-1,0.5)$) {$x_1$};
  \node [observed,name=y] at ($(-1,-0.5)$) {$y$};
  \node [observed,name=x2] at ($(1,0.5)$) {$x_2$};
  \node [observed,name=x] at ($(1,-0.5)$) {$x$};
  \node [hidden,name=u] at ($(-1/3,0)$) {$u$};
  \node [hidden,name=u2] at ($(1/3,0)$) {$u_2$};
  \draw [-] (u) to (x1);
  \draw [-] (u) to (y);
  \draw [-] (u) to (u2);
  \draw [-] (u2) to (x2);
  \draw [-] (u2) to (x);
\end{tikzpicture}
\end{center}
Therefore, by Lemma~\ref{lemma:transfer},
\[ \sigma_k(\E[x_1x^\top]) \geq \frac{\sigma_k(\E[x_1x_2^\top]) \cdot
\sigma_k(\E[yx^\top])}{\sigma_1(\E[yx_2^\top])}
\geq \frac{\gammamin \cdot (1-\veps) \theta}{\gammamax}
= \varsigma
.
\]
This gives the first claimed inequality;
now we show the second.
If $x_3 = y$, then $\sigma_k(\E[x_1y^\top]) = \sigma_k(\E[x_1x_3^\top])
\geq \gammamin \geq \varsigma$.
If $x_3 \neq y$, then we have the following induced (undirected) topology.
\begin{center}
\begin{tikzpicture}
  [
    scale=1.0,
    observed/.style={circle,minimum size=0.5cm,inner sep=0mm,draw=black,fill=black!20},
    hidden/.style={circle,minimum size=0.5cm,inner sep=0mm,draw=black}
  ]
  \node [observed,name=x1] at ($(-1,0.5)$) {$x_1$};
  \node [observed,name=x] at ($(-1,-0.5)$) {$x$};
  \node [observed,name=x3] at ($(1,0.5)$) {$x_3$};
  \node [observed,name=y] at ($(1,-0.5)$) {$y$};
  \node [hidden,name=u] at ($(-1/3,0)$) {$u$};
  \node [hidden,name=u3] at ($(1/3,0)$) {$u_3$};
  \draw [-] (u) to (x1);
  \draw [-] (u) to (x);
  \draw [-] (u) to (u3);
  \draw [-] (u3) to (x3);
  \draw [-] (u3) to (y);
\end{tikzpicture}
\end{center}
Again, by Lemma~\ref{lemma:transfer},
\[ \sigma_k(\E[x_1y^\top]) \geq \frac{\sigma_k(\E[x_1x_3^\top]) \cdot
\sigma_k(\E[xy^\top])}{\sigma_1(\E[xx_3^\top])}
\geq \frac{\gammamin \cdot (1-\veps) \theta}{\gammamax}
= \varsigma
.
\qedhere
\]
\end{proof}
Claim~\ref{claim:min-correlation}, Lemma~\ref{lem:usefulrankk}, and the
sample size requirement of Theorem~\ref{theorem:rg} (as per
\eqref{eq:min-correlation-requirement}) imply that the spectral quartet
test on $\{x,x_1,y,y_1\}$ returns the correct pairing.
Since the induced (undirected) topology is
\begin{center}
\begin{tikzpicture}
  [
    scale=1.0,
    observed/.style={circle,minimum size=0.5cm,inner sep=0mm,draw=black,fill=black!20},
    hidden/.style={circle,minimum size=0.5cm,inner sep=0mm,draw=black}
  ]
  \node [observed,name=x1] at (-1,0.5) {$x_1$};
  \node [observed,name=x] at (-1,-0.5) {$x$};
  \node [observed,name=y1] at (1,0.5) {$y_1$};
  \node [observed,name=y] at (1,-0.5) {$y$};
  \node [hidden,name=u] at ($(-1/3,0)$) {$u$};
  \node [hidden,name=v] at ($(1/3,0)$) {$v$};
  \draw [-] (x) to (u);
  \draw [-] (x1) to (u);
  \draw [-] (y) to (v);
  \draw [-] (y1) to (v);
  \draw [-] (u) to (v);
\end{tikzpicture}
\end{center}
the correct pairing is $\{\{x,x_1\},\{y,y_1\}\}$.
Because the leaf sets $\{ \Leaves[r] : r \in \Roots \}$ partition $\Vobs$,
and because $x_1 \not\in \Leaves[u]$ and $y_1 \not\in \Leaves[v]$,
there exists $\{u',v'\} \subseteq \Roots \setminus \{u,v\}$ such that
$x_1 \in \Leaves[u']$ and $y_1 \in \Leaves[v']$.
This proves the lemma in this case.

Now instead suppose (ii) holds.
Since $\Tree$ is connected, and $\Subtree[u]$ and $\Subtree[v]$ are
respectively rooted at $u$ and $v$, there must exist a pair $\{u_1,v_1\}
\subset (\Roots \setminus \{u,v\}) \cup \Vhid[\Comps]$ such that
neither $u_1$ nor $v_1$ are leaves in $\Tree$, $u_1$ is adjacent to $u$ in
$\Tree$, $v_1$ is adjacent to $v$ in $\Tree$, and the (undirected) path
from $u$ to $v$ in $\Tree$ passes through the path from $u_1$ to $v_1$.
\begin{center}
\begin{tikzpicture}
  [
    scale=1.0,
    observed/.style={circle,minimum size=0.5cm,inner sep=0mm,draw=black,fill=black!20},
    hidden/.style={circle,minimum size=0.5cm,inner sep=0mm,draw=black}
  ]
  \node [hidden,name=u] at ($(-2,0)$) {$u$};
  \node [hidden,name=u1] at ($(-1,0)$) {$u_1$};
  \node [hidden,name=v1] at ($(1,0)$) {$v_1$};
  \node [hidden,name=v] at ($(2,0)$) {$v$};
  \draw [-] (u) to (u1);
  \draw [loosely dashed] (u1) to (v1);
  \draw [-] (v1) to (v);
\end{tikzpicture}
\end{center}
An argument analogous to that in case (i) applies to prove the lemma in
this case; we provide a brief sketch below.
Because $u_1$ is not a leaf, there exists three subtrees
$\{\Subtree_{u_1,1}, \Subtree_{u_1,2}, \Subtree_{u_1,3}\} \subseteq
\Forest_{u_1}$ such that $u$ is the root of $\Subtree_{u_1,2}$ (so $x \in
\Vobs[\Subtree_{u_1,2}]$) and $y \in \Vobs[\Subtree_{u_1,3}]$.
Moreover, there exist
$x_1 \in \Vobs[\Subtree_{u_1,1}]$,
$x_2 \in \Vobs[\Subtree_{u_1,2}]$, and
$x_3 \in \Vobs[\Subtree_{u_1,3}]$ such that $\sigma_k(\E[x_ix_j^\top]) \geq
\gammamin$ for all $\{i,j\} \subset \{1,2,3\}$
(it is possible to have $x_2 = x$ and $x_3 = y$).
Let $u_1'$ denote the root of $\Subtree_{u_1,1}$, $u_2'$ denote the node in
$\Subtree_{u_1,2}$ at which the (undirected) paths $x \leadsto u_1$ and
$x_2 \leadsto u_1$ intersect (if $x_2 = x$, then let $u_2' = u$, which is
the root of $\Subtree_{u_1,2}$), and $u_3$ denote the node in
$\Subtree_{u_1,2}$ at which the (undirected) paths $y \leadsto u_1$ and
$x_3 \leadsto u_1$ intersect (if $x_3 = y$, then let $u_3$ be the root of
$\Subtree_{u_1,3}$).
An analogous argument applies relative to $v_1$ instead of $u_1$; the
induced (undirected) topologies are given below.
\begin{center}
\begin{tikzpicture}
  [
    scale=1.0,
    observed/.style={circle,minimum size=0.5cm,inner sep=0mm,draw=black,fill=black!20},
    hidden/.style={circle,minimum size=0.5cm,inner sep=0mm,draw=black},
  ]
  \node [hidden,name=u1] at ($(0,0)$) {$u_1$};
  \node [hidden,name=u1p] at ($(-2,-1)$) {$u_1'$};
  \node [hidden,name=u2p] at ($(0,-1)$) {$u_2'$};
  \node [hidden,name=u3p] at ($(2,-1)$) {$u_3'$};
  \node [observed,name=x] at ($(0.5,-2)$) {$x$};
  \node [observed,name=y] at ($(2.5,-2)$) {$y$};
  \node [observed,name=x1] at ($(-2,-2)$) {$x_1$};
  \node [observed,name=x2] at ($(-0.5,-2)$) {$x_2$};
  \node [observed,name=x3] at ($(1.5,-2)$) {$x_3$};
  \draw [-] (u1) to (u1p);
  \draw [-] (u1) to (u2p);
  \draw [-] (u1) to (u3p);
  \draw [-] (u1p) to (x1);
  \draw [-] (u2p) to (x2);
  \draw [-] (u2p) to (x);
  \draw [-] (u3p) to (x3);
  \draw [-] (u3p) to (y);
\end{tikzpicture}
\qquad
\begin{tikzpicture}
  [
    scale=1.0,
    observed/.style={circle,minimum size=0.5cm,inner sep=0mm,draw=black,fill=black!20},
    hidden/.style={circle,minimum size=0.5cm,inner sep=0mm,draw=black},
  ]
  \node [hidden,name=v1] at ($(0,0)$) {$v_1$};
  \node [hidden,name=v1p] at ($(-2,-1)$) {$v_1'$};
  \node [hidden,name=v2p] at ($(0,-1)$) {$v_2'$};
  \node [hidden,name=v3p] at ($(2,-1)$) {$v_3'$};
  \node [observed,name=y] at ($(0.5,-2)$) {$y$};
  \node [observed,name=x] at ($(2.5,-2)$) {$x$};
  \node [observed,name=y1] at ($(-2,-2)$) {$y_1$};
  \node [observed,name=y2] at ($(-0.5,-2)$) {$y_2$};
  \node [observed,name=y3] at ($(1.5,-2)$) {$y_3$};
  \draw [-] (v1) to (v1p);
  \draw [-] (v1) to (v2p);
  \draw [-] (v1) to (v3p);
  \draw [-] (v1p) to (y1);
  \draw [-] (v2p) to (y2);
  \draw [-] (v2p) to (y);
  \draw [-] (v3p) to (y3);
  \draw [-] (v3p) to (x);
\end{tikzpicture}
\end{center}
Using the arguments in Claim~\ref{claim:min-correlation}, it can be shown
that the inequalities in~\eqref{eq:min-correlation} hold in this case, so
by Lemma~\ref{lem:usefulrankk}, the quartet test on $\{x,x_1,y,y_1\}$
returns $\{\{x,x_1\},\{y,y_1\}\}$.
Because the leaf sets $\{ \Leaves[r] : r \in \Roots \}$ partition $\Vobs$,
and because $x_1 \not\in \Leaves[u] = \Vobs[\Subtree_{u_1,2}]$ and $y_1
\not\in \Leaves[v] = \Vobs[\Subtree_{v_1,2}]$,
there exists $\{u',v'\} \subseteq \Roots \setminus \{u,v\}$ such that
$x_1 \in \Leaves[u']$ and $y_1 \in \Leaves[v']$.
This proves the lemma in this case.

Finally, suppose (iii) holds.
Without loss of generality, assume $\Subtree[u]$ is not a leaf component
relative to $\Comps$.
Since $\Tree$ is connected, and $\Subtree[u]$ and $\Subtree[v]$ are
respectively rooted at $u$ and $v$, there must exist $v_1 \in (\Roots
\setminus \{u,v\}) \cup \Vhid[\Comps]$ such that $v_1$ is not a leaf in
$\Tree$, $v_1$ is adjacent to $v$ in $\Tree$, and the (undirected) path
from $u$ to $v$ in $\Tree$ passes through $v_1$.
Moreover, since $\Subtree[u]$ is not a leaf component relative to $\Comps$,
it has degree $\geq2$ in $\Supertree[\Comps]$.
Note that $u$ is not a leaf in $\Tree$, and moreover, there exists $u_1 \in
(\Roots \setminus \{u,v\}) \cup \Vhid[\Comps]$ such that $u_1$ is adjacent
to $u$ in $\Tree$, and $u_1$ is not on the (undirected) path from $u$ to
$v$.
\begin{center}
\begin{tikzpicture}
  [
    scale=1.0,
    observed/.style={circle,minimum size=0.5cm,inner sep=0mm,draw=black,fill=black!20},
    hidden/.style={circle,minimum size=0.5cm,inner sep=0mm,draw=black}
  ]
  \node [hidden,name=u] at ($(-1,0)$) {$u$};
  \node [hidden,name=u1] at ($(-2,0)$) {$u_1$};
  \node [hidden,name=v1] at ($(1,0)$) {$v_1$};
  \node [hidden,name=v] at ($(2,0)$) {$v$};
  \draw [-] (u) to (u1);
  \draw [loosely dashed] (u) to (v1);
  \draw [-] (v1) to (v);
\end{tikzpicture}
\end{center}
Again, an argument analogous to that in case (i) applies now to prove the
lemma in this case.
\end{proof}

Finally, we give a lemma which analyzes the while-loop of
Algorithm~\ref{alg:rg} and consequently implies Theorem~\ref{theorem:rg}.
\begin{lemma}[Loop invariants] \label{lemma:loop-invariant}
The following invariants concerning the state of the objects
$(\Roots,\Subtree[\cdot],\Leaves[\cdot])$ hold before the while-loop in
Algorithm~\ref{alg:rg}, and after each iteration of the while-loop.
\begin{enumerate}
\item $\Roots \subseteq \Vars$, and for each $u \in \Roots$, $\Subtree[u]$
is a subtree of $\Tree$ rooted at $u$.
Moreover, the rooted subtree $\Subtree[v]$ is already defined by
Algorithm~\ref{alg:rg} for every node $v$ appearing in $\Subtree[u]$ for
some $u \in \Roots$.
Finally, for each $u \in \Roots$, the subtree $\Subtree[u]$ is formed by
joining the subtrees $\Subtree[v]$ corresponding to children $v$ of $u$ in
$\Subtree[u]$ via edges $\{u,v\}$.

\item The subtrees in $\Comps := \{ \Subtree[u] : u \in \Roots \}$ are
disjoint, and the leaf sets $\{ \Leaves[u] : u \in \Roots \}$ partition
$\Vobs$.

%
\end{enumerate}
Moreover, no iteration of the while-loop terminates in failure.
\end{lemma}
Before proving Lemma~\ref{lemma:loop-invariant}, we show how it implies
Theorem~\ref{theorem:rg}.
Initially, $|\Roots| = n$, and each iteration of the while-loop decreases
the cardinality of $\Roots$ by one, so there are a total of $n-1$
iterations of the while-loop.
By Lemma~\ref{lemma:loop-invariant}, the final iteration results in a set
$\Roots = \{ h \}$ such that $\wh\Tree = \Subtree[h]$ is a subtree of
$\Tree$ rooted at $h$, and $\Leaves[h] = \Vobs$.
This implies that $\wh\Tree$ has the same (undirected) structure as
$\Tree$, as required.
This completes the proof of Theorem~\ref{theorem:rg}.

\begin{proof}[Proof of Lemma~\ref{lemma:loop-invariant}]
The loop invariants clearly hold before the while-loop with the initial
settings of $\Roots = \Vobs$, $\Subtree[x] = \text{rooted single-node tree
$x$}$, and $\Leaves[x] = \{x\}$ for all $x \in \Roots$.
So assume as the inductive hypothesis that the loop invariants hold at the
start of a particular iteration (in which $|\Roots| > 1$).
It remains to prove that the iteration does not terminate in failure, and
that the loop invariants hold at the end of the iteration.
Let $\Roots$, $\Subtree[\cdot]$, and $\Leaves[\cdot]$ be in their state at
the beginning of the iteration.

Because the second loop invariant holds, Lemma~\ref{lemma:supertree} implies
that the nodes of $\Supertree[\Comps]$ are $\Comps \cup \Vhid[\Comps]$, and
that each leaf in $\Supertree[\Comps]$ is a subtree $\Subtree[u] \in \Comps$
(so we may refer to the leaves of $\Supertree[\Comps]$ as leaf components).
\begin{claim} \label{claim:goodpair}
If $|\Roots| > 1$, then there exists a pair $\{u,v\} \subseteq \Roots$ such
that the following hold.
\begin{enumerate}
\item Either $u$ and $v$ are neighbors in $\Tree$, and at least one of
$\Subtree[u]$ or $\Subtree[v]$ is a leaf component relative to $\Comps$; or
$u$ and $v$ share a common neighbor in $\Vhid[\Comps]$, and both
$\Subtree[u]$ and $\Subtree[v]$ are leaf components relative to $\Comps$.

\item $\Mergeable(\Roots,\Leaves[\cdot],u,v) = \text{true}$.

\item $\max\{ \sigma_k(\wh\Sigma_{x,y}) : (x,y) \in \Leaves[u] \times
\Leaves[v] \} \geq \theta$.

\end{enumerate}
\end{claim}
\begin{proof}
Suppose there are no pairs $\{u,v\} \subseteq \Comps$ such that $u$ and $v$
are neighbors in $\Tree$ and at least one of $\Subtree[u]$ and
$\Subtree[v]$ is a leaf component relative to $\Comps$.
Then each leaf component must be adjacent to some $h \in \Vhid[\Comps]$ in
$\Supertree[\Comps]$.
Consider the tree $\Supertree'$ obtained from $\Supertree[\Comps]$ by
removing all the leaf components in $\Supertree[\Comps]$.
The leaves of $\Supertree'$ must be among the $h \in \Vhid[\Comps]$ that
were adjacent to the leaf components in $\Supertree[\Comps]$.
Fix such a leaf $h$ in $\Supertree'$, and observe that it has degree one in
$\Supertree'$.
By assumption, no node in $\Tree$ has degree two, so $h$ must have been
connected to at least two leaf components in $\Supertree[\Comps]$, say
$\Subtree[u]$ and $\Subtree[v]$.
The node $h$ is therefore a common neighbor of $u$ and $v$.
This proves the existence of a pair $\{u,v\} \subseteq \Roots$ satisfying
the first required property.

Fix the pair $\{u,v\}$ specified above.
By Lemma~\ref{lemma:mergeable}, $\Mergeable(\Roots,\Leaves[\cdot],u,v)$
returns true, so $\{u,v\}$ satisfies the second required property.

To show the final required property, we consider two cases.
Suppose first that $u$ and $v$ are neighbors, and that $\Subtree[u]$ is a
leaf component relative to $\Comps$.
Note that $u$ and $v$ cannot both be leaves in $\Tree$.
If $v$ is not a leaf, then there exists subtrees $\Subtree_{v,1}$ and
$\Subtree_{v,2}$ in $\Forest_v$ such that $\Subtree_{v,1} = \Subtree[u]$
(because $\Subtree[u]$ is a leaf component) and $\Subtree_{v,2} =
\Subtree[v']$ for some child $v'$ of $v$ in $\Subtree[v]$ (by the first
loop invariant).
By Condition~\ref{cond:correlation}, there exists $x \in
\Vobs[\Subtree_{v,1}] = \Leaves[u]$ and $y \in \Vobs[\Subtree_{v,2}]
\subseteq \Leaves[v]$ such that $\sigma_k(\Sig_{x,y}) \geq \gammamin =
(1+\veps)\theta$; by Lemma~\ref{lemma:weyl-application},
$\sigma_k(\wh\Sig_{x,y}) \geq \theta$.
If $v$ is a leaf but $u$ is not, then there exists subtrees
$\Subtree_{u,1}$ and $\Subtree_{u,2}$ in $\Forest_u$ such that
$\Subtree_{u,1} = v$ and $\Subtree_{u,2} = \Subtree[u']$ for some child
$u'$ of $u$ in $\Subtree[u]$ (by the first loop invariant).
So by Condition~\ref{cond:correlation}, there $y \in \Vobs[\Subtree_{u,2}]
\subseteq \Leaves[u]$ such that $\sigma_k(\Sig_{v,y}) \geq \gammamin =
(1+\veps)\theta$; by Lemma~\ref{lemma:weyl-application},
$\sigma_k(\wh\Sig_{v,y}) \geq \theta$.
Now instead suppose that $u$ and $v$ share a common neighbor $h$, and that
both $\Subtree[u]$ and $\Subtree[v]$ are leaf components relative to
$\Comps$.
This latter fact implies that $\{ \Subtree[u], \Subtree[v] \} \subset
\Forest_h$, so Condition~\ref{cond:correlation} implies that there exists
$x \in \Vobs[\Subtree[u]] = \Leaves[u]$ and $y \in \Vobs[\Subtree[v]] =
\Leaves[v]$ such that $\sigma_k(\Sig_{x,y}) \geq \gammamin =
(1+\veps)\theta$.
By Lemma~\ref{lemma:weyl-application}, $\sigma_k(\wh\Sig_{x,y}) \geq
\theta$.
\end{proof}

\begin{claim} \label{claim:badpair}
Consider any pair $\{u,v\} \subseteq \Roots$ such that $\max\{
\sigma_k(\wh\Sigma_{x,y}) : (x,y) \in \Leaves[u] \times \Leaves[v] \} \geq
\theta$.
If the first property from Claim~\ref{claim:goodpair} fails to hold for
$\{u,v\}$, then $\Mergeable(\Roots,\Leaves[\cdot],u,v) = \text{false}$.
\end{claim}
\begin{proof}
This follows immediately from Lemma~\ref{lemma:not-mergeable}.
\end{proof}
Taken together, Claims~\ref{claim:goodpair} and~\ref{claim:badpair} imply
that the pair $\{u,v\} \subseteq \Roots$ selected by the first step in the
while-loop indeed exists (so the iteration does not terminate in failure)
and satisfies the properties in Claim~\ref{claim:goodpair}.

Now we consider the second step of the while-loop, which is the call to the
subroutine $\Relationship$.
\begin{claim} \label{claim:relationship}
Suppose a pair $\{u,v\}$ satisfies the properties in
Claim~\ref{claim:goodpair}.
Then $\Relationship(\Roots,\Leaves[\cdot],\Subtree[\cdot],u,v)$ returns the
correct relationship for $u$ and $v$.
Specifically:
\begin{enumerate}
\item If $u$ and $v$ share a common neighbor in $\Tree$ (and both are leaf
components relative to $\Comps$), then ``siblings'' is returned.

\item If $u$ and $v$ are neighbors in $\Tree$ and $\Subtree[v]$ is a leaf
component relative to $\Comps$ but $\Subtree[u]$ is not, then
``$u$ is parent of $v$'' is returned.

\item If $u$ and $v$ are neighbors in $\Tree$ and $\Subtree[u]$ is a leaf
component relative to $\Comps$ but $\Subtree[v]$ is not, then
``$v$ is parent of $u$'' is returned.

\item If $u$ and $v$ are neighbors in $\Tree$ and both $\Subtree[u]$ and
$\Subtree[v]$ are leaf components relative to $\Comps$, and $u$ is a leaf
in $\Tree$ but $v$ is not, then ``$v$ is parent of $u$'' is returned.

\item If $u$ and $v$ are neighbors in $\Tree$ and both $\Subtree[u]$ and
$\Subtree[v]$ are leaf components relative to $\Comps$, and $v$ is a leaf
in $\Tree$ but $u$ is not, then ``$u$ is parent of $v$'' is returned.

\item If $u$ and $v$ are neighbors in $\Tree$ and both $\Subtree[u]$ and
$\Subtree[v]$ are leaf components relative to $\Comps$, and neither $u$ nor
$v$ are leaves in $\Tree$, then ``$u$ is parent of $v$'' is returned.

\end{enumerate}
\end{claim}
\begin{proof}
Fix the pair $(x,y) \in \Leaves[u] \times \Leaves[v]$ guaranteed by the
third property of Claim~\ref{claim:goodpair} such that
$\sigma_k(\wh\Sig_{x,y}) \geq \theta$.
Now we consider the possible relationships between $u$ and $v$.

Suppose $u$ and $v$ share a common neighbor $h \in \Vhid[\Comps]$ in
$\Tree$, and that both $\Subtree[u]$ and $\Subtree[v]$ are leaf components
relative to $\Comps$.
We need to show that the subroutine $\Relationship$ asserts both ``$u
\not\to v$'' and ``$v \not\to u$''.
To show that ``$u \not\to v$'' is asserted, we assume $u$ is not a leaf
(otherwise ``$u \not\to v$'' is immediately asserted and we're done), let
$\{u_1,\dotsc,u_q\}$ be the children of $u$ in $\Subtree[u]$, and take
$\Roots[u]$ as defined in $\Relationship$.
By the first loop invariant, the subtrees in $\Comps[u]$ are disjoint, and
the leaf sets $\{ \Leaves[r] : r \in \Roots[u] \}$ partition $\Vobs$.
In particular, $x \in \Leaves[u_i]$ for some $i \in \{1,\dotsc,q\}$.
Since $u_i$ and $v$ are not neighbors, and do not share a common neighbor.
Therefore, by Lemma~\ref{lemma:not-mergeable},
$\Mergeable(\Roots[u],\Leaves[\cdot],u_i,v) = \text{false}$, so ``$u
\not\to v$'' is asserted.
A similar argument implies that ``$v \not\to u$'' is asserted.
Since both ``$u \not\to v$'' and ``$v \not\to u$'' are asserted, the
subroutine returns ``siblings''.

Now instead suppose $u$ and $v$ are neighbors.
First, suppose $\Subtree[u]$ is a leaf component relative to $\Comps$.
We claim that if $v$ is not a leaf, then ``$v \not\to u$'' is not asserted.
Let $\{v_1,\dotsc,v_q\}$ be the children of $v$ in $\Subtree[v]$, and take
$\Roots[v] = \{ u, v_1, \dotsc, v_q \}$ as defined in $\Relationship$.
By the first loop invariant, the subtrees in $\Comps[v]$ are disjoint, and
the leaf sets $\{ \Leaves[r] : r \in \Roots[v] \}$ partition $\Vobs$.
By Lemma~\ref{lemma:leaf-components}, $\Subtree[u]$ and $\Subtree[v_i]$ are
leaf components relative to $\Comps[v]$ for each $i \in \{1,\dotsc,q\}$.
For each $i \in \{1,\dotsc,q\}$, $\{u,v_i\}$ share $v$ as a common
neighbor, and $\Subtree[u]$ and $\Subtree[v_i]$ are both leaf components
relative to $\Comps[v]$.
Therefore by Lemma~\ref{lemma:mergeable},
$\Mergeable(\Roots[v],\Leaves[\cdot],u,v_i) = \text{true}$ for all $i \in
\{1,\dotsc,q\}$, so ``$v \not\to u$'' is not asserted.

Suppose $\Subtree[u]$ is a leaf component relative to $\Comps$ but
$\Subtree[v]$ is not.
By Lemma~\ref{lemma:supertree}, $v$ is not a leaf in $\Tree$, so as argued
above, ``$v \not\to u$'' is not asserted.
It remains to show that ``$u \not\to v$'' is asserted.
Assume $u$ is not a leaf (or else $u \not\to v$ is immediately asserted and
we're done), let $\{u_1,\dotsc,u_q\}$ be the children of $u$ in
$\Subtree[u]$, and take $\Roots[u]$ as defined in $\Relationship$.
By the first loop invariant, the subtrees in $\Comps[u]$ are disjoint, and
the leaf sets $\{ \Leaves[r] : r \in \Roots[u] \}$ partition $\Vobs$.
In particular, $x \in \Leaves[u_i]$ for some $i \in \{1,\dotsc,q\}$.
By Lemma~\ref{lemma:leaf-components}, $\Subtree[v]$ is not a leaf component
relative to $\Comps[u]$.
Moreover, $u_i$ and $v$ are not neighbors.
Therefore by Lemma~\ref{lemma:not-mergeable},
$\Mergeable(\Roots[u],\Leaves[\cdot],u_i,v) = \text{false}$, so ``$u
\not\to v$'' is asserted.
Since ``$v \not\to u$'' is not asserted but ``$u \not\to v$'' is asserted,
the subroutine returns ``$v \to u$''.
An analogous argument shows that if $\Subtree[v]$ is a leaf component
relative to $\Comps$ but $\Subtree[u]$ is not, then the subroutine returns
``$u \to v$''.

Now suppose both $\Subtree[u]$ and $\Subtree[v]$ are leaf components
relative to $\Comps$.
By assumption, leaves in $\Tree$ are only adjacent to non-leaves, so it
cannot be that both $u$ and $v$ are leaves.
Therefore at least one of $u$ and $v$ is not a leaf in $\Tree$.
Without loss of generality, say $v$ is not a leaf in $\Tree$.
Then as argued above, ``$v \not\to u$'' is not asserted.
If $u$ is a leaf, then ``$u \not\to v$'' is asserted, so the subroutine
returns ``$v \to u$''.
If $u$ is not a leaf, then by symmetry, ``$u \not\to v$'' is not asserted.
Therefore the subroutine returns ``$u \to v$''.
\end{proof}
Claim~\ref{claim:relationship} implies that the remaining steps in the
while-loop after the call to $\Relationship$ preserve the two loop
invariants, simply by construction.
\end{proof}

There is one last lemma used in the proof of
Lemma~\ref{lemma:loop-invariant}.
\begin{lemma}[Leaf components] \label{lemma:leaf-components}
Suppose the invariants in Lemma~\ref{lemma:loop-invariant} are satisfied.
Then for each $u \in \Roots$ such that $u$ is not a leaf in $\Tree$, the
leaf components relative to the collection
\[ \Comps[u]
:= (\Comps \setminus \{ \Subtree[u] \})
\ \cup \
\{ \Subtree[v] : \text{$v$ is a child of $u$ in $\Subtree[u]$} \}
\]
are
\[
\{ \Subtree[r] : r \neq u \ \wedge \
\text{$\Subtree[r]$ is a leaf component relative to $\Comps$} \}
\ \cup \
\{ \Subtree[r] : \text{$r$ is a child of $u$ in $\Subtree[u]$} \}
.
\]
\end{lemma}
\begin{proof}
Pick any $u \in \Roots$ such that $u$ is not a leaf in $\Tree$.
Let $\{v_1,\dotsc,v_q\}$ be the children of $u$ in $\Subtree[u]$.
By the first loop invariant, each $v_i$ is the
root of a subtree $\Subtree[v_i]$.
This implies that the subtrees $\{ \Subtree[v_1], \dotsc, \Subtree[v_q] \}$
are disjoint and $\{ \Leaves[v_1], \dotsc, \Leaves[v_q] \}$ partition
$\Leaves[u]$.
Therefore $\Supertree[\Comps[u]]$ is the same as $\Supertree[\Comps]$
except with the following changes.
\begin{enumerate}
\item $\Subtree[u]$ is replaced with $u$.

\item For each $i$, $\Subtree[v_i]$ is added with the edge $\{u,v_i\}$.

\end{enumerate}
This means that each $\Subtree[v_i]$ has degree one in
$\Supertree[\Comps[u]]$ and therefore is a leaf component relative to
$\Comps[u]$.
\end{proof}

\end{document}